\documentclass[preprint,3p,times]{elsarticle}

\usepackage{hyperref,txfonts}
\usepackage{amssymb}
\usepackage{stmaryrd,enumitem,graphicx} 

\usepackage{amsmath}

\newtheorem{THEOREM}{Theorem}
\newenvironment{theorem}{\begin{THEOREM}  {\bf } }%
                        {\end{THEOREM}}

\newtheorem{LEMMA}[THEOREM]{Lemma}
\newenvironment{lemma}{\begin{LEMMA}  {\bf } }%
                      {\end{LEMMA}}
\newtheorem{COROLLARY}[THEOREM]{Corollary}
\newenvironment{corollary}{\begin{COROLLARY}  {\bf } }%
                          {\end{COROLLARY}}
\newtheorem{PROPOSITION}[THEOREM]{Proposition}
\newenvironment{proposition}{\begin{PROPOSITION}  {\bf } }%
                            {\end{PROPOSITION}}
\newtheorem{DEFINITION}[THEOREM]{Definition}
\newenvironment{definition}{\begin{DEFINITION}  {\bf } \rm}%
                            { \end{DEFINITION} \vspace{.1in}}
	
\newenvironment{ldefn}{\begin{DEFINITION}  {\bf } \rm}%
                            {\end{DEFINITION} \vspace{.1in}}

\newtheorem{CLAIM}[THEOREM]{Claim}
                            {\end{CLAIM}}
\newtheorem{EXAMPLE}[THEOREM]{Example}
                            {\end{EXAMPLE}}
\newtheorem{REMARK}[THEOREM]{Remark}
                            {\end{REMARK}}
\newtheorem{NOTATION}[THEOREM]{Notation}
							                            {\end{NOTATION}}
							
\newenvironment{proof}{\noindent {\bf Proof:} \hspace{.2em}}%
                     {}

					                     {}

					                     {}

\newcommand{\bbox}{\vrule height7pt width4pt depth1pt}
\newcommand{\is}{\exists}
\newcommand{\disp}{\displaystyle}
\newcommand{\dis}{\disp}
\renewcommand{\qed}{\bbox\vspace{0.1in}}

\newcommand{\fullintegral}{{\int_{-\infty}^{\infty}}}

\newcommand{\fullintegralx}{\int_{\vec x}}
\newcommand{ \fullx}{\fullintegralx}

\newcommand{\now}{\ti{now}}

\newcommand{\ti}{\textit}

\newcommand{\stara}{\textrm{\upshape{P3}}}%
\newcommand{\probabbr}[3]{\lan #1.#2\to #3 \ran}%
\newcommand{\dummy}[1]{\textrm{d}{#1}}

\newcommand{\xpos}{\textit{h}}%
\newcommand{\ypos}{\textit{v}}%

\newcommand{\xmove}{\ti{move}}%
\newcommand{ \move}{\xmove}

\newcommand{\ymove}{\textit{up}}

\newcommand{\lan}{\langle}
\newcommand{\ran}{\rangle}
\newcommand{\p}{p}%

\newcommand{\poss}{\textit{Poss}}
\newcommand{ \golog}{\textit{Golog}{}}
\newcommand{ \alt}{\ti{Alt}}
\newcommand{\prettysum}[1]{\sum \raisebox{-7pt}{${\!}_{{#1}}$}}

\newcommand{\sub}{_}
\def\su{^}

\newcommand{\natt}{\textit{Natural}}
\newcommand{\summ}{\textsc{SUM}}

\newcommand{\real}{{\mathbb{R}}}

\newcommand{\nat}{{\mathbb{N}}}

\newcommand{\defeq}{\doteq}%

\newcommand{\bel}{\textit{Bel}}
\newcommand{\func}{\textit{P}}
\newcommand{\ivar}{\iota}%
\newcommand{\ins}{S_{\!\textrm{0}}}

\newcommand{\init}{\textit{Init}}%

\newcommand{\xsense}{\textit{sonar}}%
\newcommand{\sonar}{\xsense}
\newcommand{\chw}{\ti{setwin}}

\newcommand{\seew}{\ti{seewin}}

\newcommand{\prox}{\ti{sensewall}}
\newcommand{\close}{\ti{close}}
\newcommand{\far}{\ti{far}}

\newcommand{\Err}{\ti{Err}}

\newcommand{\win}{\ti{w}}%

\newcommand{\lt}{<}
\newcommand{\gt}{>}
\newcommand{\all}{\forall}
\newcommand{\thereis}{\exists}
\newcommand{\infinity}{\infty}

\newcommand{\D}{{\cal D}}

\newcommand{\N}{{\cal N}}

\newcommand{\xO}{{x_1}}
\newcommand{\xT}{{x_2}}
\newcommand{ \xo}{\xO}
\newcommand{ \xt}{\xT}

\newcommand{\abs}[1]{\left| #1\right|}
\newcommand{\set}[1]{\left\{ #1 \right\}}

\newcommand{\la}{\langle}
\newcommand{\ra}{\rangle}

\newcommand{\ie}{\emph{i.e.,}~}

\newcommand{\cf}{\emph{cf.~}}

\newcommand{\know}{\textit{Knows}}%

\renewcommand{\L}{\mathcal{L}}

\usepackage{amssymb,amsmath}

\usepackage{times}

\newcommand{\true}{{\textit{true}}}%
\newcommand{\false}{\mbox{{\it false}}}

\newcommand{\commentout}[1]{}

\newcommand{\blemma}{\begin{lemma}}
\newcommand{\elemma}{\end{lemma}}

\newcommand{\bthm}{\begin{theorem}}
\newcommand{\ethm}{\end{theorem}}
\newcommand{\bprf}{\begin{proof}}
\newcommand{\eprf}{\end{proof}}
\newcommand{\bprop}{\begin{proposition}}
\newcommand{\eprop}{\end{proposition}}

\newcommand{\bi}{\begin{itemize}}
\newcommand{\ei}{\end{itemize}}
\newcommand{\be}{\begin{enumerate}}
\newcommand{\ee}{\end{enumerate}}
\newcommand{\beq}{\begin{equation}}
\newcommand{\eeq}{\end{equation}}
\newcommand{\bcase}{\begin{cases}}
\newcommand{\ecase}{\end{cases}}

\begin{document}

\begin{frontmatter}

\title{Reasoning about Discrete and Continuous \\ Noisy Sensors and Effectors in Dynamical Systems}

\cortext[cor]{Corresponding author.}
 
 \address[edinburgh]{University of Edinburgh,  
Edinburgh, United Kingdom.}
\address[turing]{The Alan Turing Institute, London, United Kingdom.}
\address[toronto]{University of Toronto, Toronto,  Canada.}

\author[edinburgh,turing]{Vaishak Belle\corref{cor}}
\ead{vaishak@ed.ac.uk}   
\author[toronto]{Hector J.~Levesque}
\ead{hector@cs.toronto.edu}

\begin{abstract}
	Among the many approaches for reasoning about degrees of belief in 
	the presence of noisy sensing and acting,  the logical account
	proposed by Bacchus, Halpern, and Levesque is perhaps the most expressive.
	While their formalism is quite general, it is restricted to fluents
	whose values are drawn from discrete finite domains, as opposed to
	the continuous domains seen in many robotic applications. In this
	work, we show how this limitation in that approach can be lifted.
	By dealing seamlessly with both discrete distributions and continuous
	densities within a rich theory of action, we provide a very general
	logical specification of how belief should change after acting and
	sensing in complex noisy domains.
\end{abstract}


\end{frontmatter}

\section{Introduction} %
\label{sec:introduction}

\begin{quotation} \it 
On numerous occasions it has been suggested that the formalism [the situation calculus] take uncertainty into account by attaching probabilities to its sentences. We agree that the formalism will eventually have to allow statements about the probabilities of events, but attaching probabilities to all statements has the following objections: \begin{enumerate}
	\item It is not clear how to attach probabilities to statements containing quantifiers in a way that corresponds to the amount of conviction people have.
	\item The information necessary to assign numerical probabilities is not ordinarily available. Therefore, a formalism that required numerical probabilities would be epistemologically inadequate.
\end{enumerate}
\hfill \( - \) McCarthy and  Hayes \cite{McCarthy:69}. 

\end{quotation}

Much of high-level AI research is concerned with the behaviour of some  putative agent, such as an autonomous robot, operating in an environment. Broadly speaking, an intelligent agent  interacting with a dynamic and incompletely known world grapples with two special sorts of reasoning problems. First, because the world is \textit{dynamic}, it will need to reason about change: how its actions  affect the state of the world. Pushing an object on a table, for example, may cause it to fall on the floor, where it will remain unless picked up. Second, because the world is incompletely known, the agent will need to make do with partial specifications about what is true. As a result, the agent will often need to augment what it believes about the world by performing perceptual actions, using sensors of one form or another.

For many AI applications, and robotics in particular, these reasoning problems are more involved. Here, 
it is not enough to deal
with incomplete knowledge, where some formula $\phi$ might be unknown. One
must also know which of $\phi$ or $\lnot\phi$ is the more {\it likely}, and by how
much. In addition, both the sensors and the effectors that the agent uses to modify its world are often subject to uncertainty in that they are {\it noisy.} 

To see a very simple example, imagine a robot moving towards a wall as shown in Figure \ref{fig:robot}, and a certain distance \( h \) from it. Suppose the robot can move towards and away from the wall, and it is equipped with a distance sensor aimed at the wall. 
Here, the robot may not know the true value of \( h \) but may believe that it takes values from some set, say \( \set{2, \ldots, 11} \). If the sensor is noisy, a reading of, say, 5 units, does not guarantee that the agent is actually 5 units from the wall, although it should serve to increase the agent's degree of belief in that fact. Analogously, if the robot intends to move by 1 unit and the effector is noisy, it may end up moving by 0.9 units, which the agent does not get to observe. Be that as it may, the robot's degree of belief that it is closer to the wall should increase.

While many proposals have appeared in the literature to address such concerns (\cf penultimate section), very few are embedded in a general theory of action whilst supporting features like disjunction and quantification. For example, graphical models such as Bayesian networks can represent and reason about the probabilistic dependencies between random variables, and how that might change over time. However, it lacks first-order features and a rich account of actions. Relational graphical models, including Markov logic networks \cite{richardson2006markov}, borrow  devices from first-order logic to allow the succinct modelling of relational dependencies, but ultimately they are purely syntactic extensions to graphical models, and do not attempt to address the deeper issues pertaining to the specification of probabilities in the presence of logical connectives and quantifiers. 
Building on first-order accounts of probabilistic reasoning \cite{bacchus1990representing,halpern1990analysis}, perhaps the most general formalism for dealing with {\it degrees of
  belief} in formulas, and in particular, with how degrees of belief should
evolve in the presence of noisy sensing and acting is the account proposed by
Bacchus, Halpern, and Levesque \cite{Bacchus1999171}, henceforth
BHL. Among its many properties, the BHL model shows precisely how beliefs can
be made less certain by acting with noisy effectors, but made more certain by
sensing (even when the sensors themselves are noisy).

The main advantage of a logical account like BHL is that it allows a
specification of belief that can be partial or incomplete, in keeping with
whatever information is available about the application domain.  It does not
require specifying a prior distribution over some random variables from which
posterior distributions are then calculated, as in Kalman filters, for example
\cite{dean1991planning}.  Nor does it require specifying the conditional
independences among random variables and how these dependencies change as the
result of actions, as in the temporal extensions to Bayesian networks
\cite{pearl1988probabilistic}.  In the BHL model, some logical constraints are
imposed on the initial state of belief. These constraints may be compatible
with one or very many initial distributions and sets of independence
assumptions. All the properties of belief will then follow at a corresponding
level of specificity.

\begin{figure}[tp]
\begin{center}
\includegraphics[width=5cm]{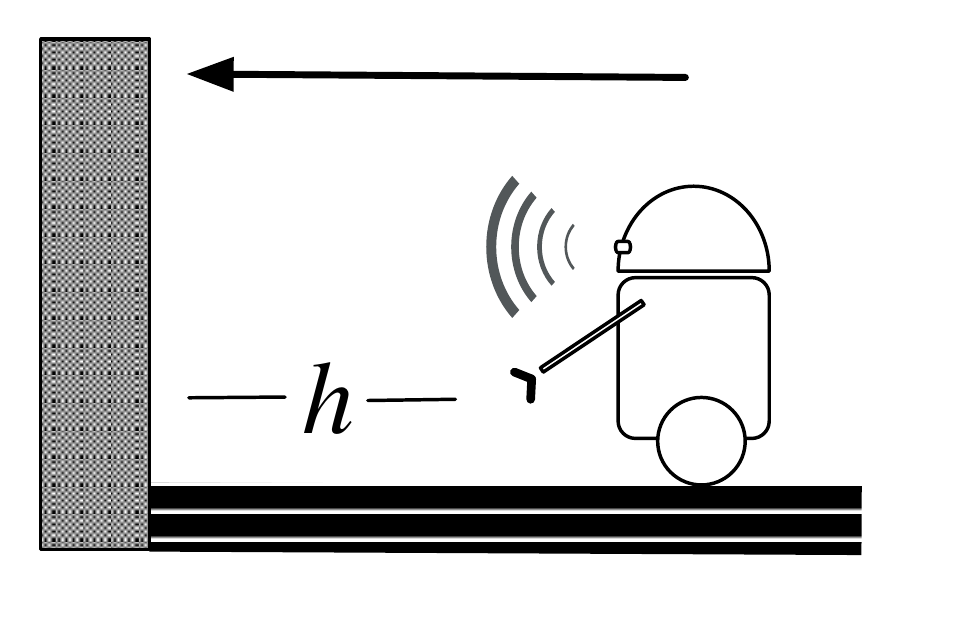}
\caption{A simple robot.}
\label{fig:robot}
\end{center}
\end{figure}

Subjective uncertainty is captured in the 
 BHL account using a possible-world model of belief
\cite{Kripke:1963uq,nla.cat-vn2250548,reasoning:about:knowledge}.  In
classical possible-world semantics, a formula $\phi$ is believed to be true
when $\phi$ holds in all possible worlds that are deemed
accessible. In BHL, the degree of belief in $\phi$ is defined as a
normalized sum over the possible worlds where $\phi$ is true of some
nonnegative {\it weights\/} associated with those worlds.  (Inaccessible
worlds are assigned a weight of zero.)  To reason about belief change, the BHL
model is then embedded in a rich theory of action and sensing provided by the
situation calculus
\cite{McCarthy:69,reiter2001knowledge,citeulike:528170}. The BHL account
provides axioms in the situation calculus regarding how the weight associated
with a possible world changes as the result of acting and sensing.  The
properties of belief and belief change  then emerge as a direct 
logical consequence of the initial constraints and these changes in weights. 

For example, suppose \( h \) is a fluent representing the
robot's horizontal distance to the wall in  Figure \ref{fig:robot}.  The fluent $\xpos$ would have
different values in different possible worlds. In a BHL specification, each of
these worlds might be given an initial weight. For example, a uniform
distribution might give an equal weight of $.1$ to ten possible worlds where
$\xpos\in\{2,3,\ldots,11\}$. The degree of belief in a formula like
$(\xpos < 9)$ is then defined as a sum of the weights, and would lead here
to a value of $.7$. The theory of action would then specify how these weights
change as the result of acting (such as moving away or towards the wall) and
sensing (such as obtaining a reading from a sonar aimed at the wall). Naturally, the logical language permits weaker specifications, involving disjunctions and quantifiers, and the appropriate behavior would still emerge.

While this model of belief is widely applicable, it does have one serious
drawback: it is ultimately based on the addition of weights and is therefore
restricted to fluents having discrete finite values. 
This is in stark contrast to robotics and machine learning applications  \cite{thrun2005probabilistic,conf/nips/ZamaniSPK12,murphy2012machine},  
where event and
observation variables are characterized by continuous
distributions, or perhaps combinations of discrete and continuous ones. 
There is no way to say in BHL that the initial value of $\xpos$ 
is any real number drawn from a uniform distribution 
on the interval \( [2,12] \).  One would
again expect the belief in $(\xpos < 9)$ to be $.7$, but instead of
being the result of {\it summing\/} weights, it must now be the result of {\it
  integrating\/} densities over a suitable space of values, something quite
beyond the  BHL approach. 

So, on the one hand, the BHL account and others like it can be seen as general formal theories that attempt to address important philosophical problems such as those raised by McCarthy and Hayes above. But on the other, a serious criticism levelled at this line of work, and indeed at much of the work in reasoning about action, is that the theory is far removed from the kind of probabilistic uncertainty  and noise seen in typical robotic applications.

The goal of this work is to show how with minimal additional assumptions this
serious limitation of BHL can be lifted.\footnote{A preliminary version of this work was discussed in \cite{bellelev13}. That work was limited to noisy sensors, but assumed deterministic (\ie noise-free) effectors.} By lifting this limitation, one obtains, for the first time, a logical  language for representing real-world robotic specifications without any modifications, but also extend beyond it by means of the logical features of the underlying framework. 
In particular, we present a \emph{formal specification}
 of the degrees of belief in formulas 
with real-valued fluents
(and other fluents too), and how belief changes as the
result of acting and sensing.
Our account will retain the advantages of 
BHL but work seamlessly with discrete probability distributions, probability
densities, and perhaps most significantly, with difficult combinations of the two. More broadly, we believe the model of belief proposed in this work provides the necessary bridge between logic-based reasoning modules, on the one hand, and probabilistic specifications as seen in real-world data-intensive applications, on the other.

The paper is organized as follows. We first review the formal
preliminaries, the BHL model in particular, and introduce  definitions for modeling continuous probability distributions. We then show how the
definition of belief in BHL can be reformulated as a different summation,
which then provides sufficient foundation for our extension to continuous
domains. We then discuss how this model can be extended for noisy acting, and combinations of discrete and continuous properties. We conclude after discussing related work.

\section{A Theory of Action} %
\label{sec:preliminaries}

Our account is formulated in the language of the situation calculus \cite{McCarthy:69}, as developed in \cite{reiter2001knowledge}. The situation calculus is a special-purpose knowledge representation formalism for reasoning about dynamical systems. 
Informally, the formalism is best understood by arranging the world in terms of three kinds of things: \emph{situations}, \emph{actions} and \emph{objects}. Situations represent ``snapshots,'' and can be viewed as possible histories. A set of initial situations correspond to the ways the world can be prior to the occurrence of actions. The result of doing an action, then, leads to a successor (non-initial) situation. Naturally, dynamic worlds  change the properties of objects, which are captured using predicates and functions whose last argument is always a situation, called \emph{fluents}. 

\subsection{The Logical Language} %
\label{sub:the_logical_language}

Formally, the language \( \L \) of the situation calculus  is a
many-sorted dialect of predicate calculus, with sorts for \emph{actions},
\emph{situations} and \emph{objects} (for everything else). (We do not review standard predicate logic here; see, for example,  \cite{enderton1972mathematical,smullyan1995first}. We further assume familiarity with the notions of \emph{models}, \emph{structures},   \emph{satisfaction} and \emph{entailment}.) 
In full length, let \( \L \) include:  
\begin{itemize} 
	\item logical connectives \( \neg, \forall, \land, = \), with other connectives such as \( \supset \) understood for their usual abbreviations; 
	\item an infinite supply of \emph{variables} of each sort; 
	\item an infinite supply of \emph{constant} symbols of the sort object;  
	\item for each \( k\geq 1, \) \emph{object} \emph{function} symbols \( g_1, g_2, \ldots \) of type \( (\mathit{action} \cup \mathit{object})^k \rightarrow \mathit{object} \); 
	\item for each \( k\geq 0,  \) \emph{action function}  symbols \( A_1, A_2, \ldots \) of type \( (\mathit{action}\cup \mathit{object})^k \rightarrow \mathit{action}; \) 
	\item a special \emph{situation} \emph{function} symbol \( \mathit{do}\): \( \mathit{action} \times \mathit{situation} \rightarrow \mathit{situation}  \); 
	\item a special predicate symbol \( \poss\): \(\mathit{action} \times \mathit{situation} \);\footnote{We will subsequently  introduce a few more distinguished predicates when modeling knowledge, sensing and nondeterminism.}
	\item for each \( k \geq 0,  \) \emph{fluent function} symbols \( f_1, f_2, \ldots \) of type \( (\mathit{action}\cup \mathit{object} )^k \times \mathit{situation} \rightarrow \mathit{object}; \) 
	\item a special constant \( \ins \) to represent the actual initial situation. 
\end{itemize}

\noindent To reiterate, apart from some  syntactic particulars, 
the logical basis for the situation calculus is the regular (many-sorted) predicate calculus.\footnote{For simplicity, only functional fluents are introduced, and their predicate counterparts are ignored. (Distinguished symbols like $\poss$ are an exception.) This is without loss of any generality since predicates can be thought of functions that take one of two values, the first denoting \( \true \) and the other denoting \( \false. \)}  So, terms and well-formed formulas are defined inductively, as usual, respecting sorts. See, for example,  \cite{reiter2001knowledge} for an exposition. 

We follow some conventions in the ways we use Latin and Greek alphabets:  $a$ for both terms and variables of the action sort (the context would make this clear);  $s$ for terms and variables of the situation sort (the context would make this clear); and finally, \( x, u, v, z, n, \) and \( y \) to range over variables of the object sort. We let $\phi$ and $\psi$ range over formulas, and $\Sigma$ over sets of formulas.  (These may be further decorated using superscripts or subscripts.) 

We sometimes suppress the situation term in a formula $\phi$, or use a distinguished 
	 variable $\mathit{now}$, to denote the current situation.  Either way, we let $\phi[t]$ denote the formula with the restored situation term $t$. 
	 
We will often use the the usual 
	``case'' notation with curly braces as a 
	convenient abbreviation for a  logical formula:
	\[
	z = \begin{cases}
	t_1 & \mbox{if\,\,} \psi\\
	t_2 & \mbox{otherwise}
	\end{cases}
	\,\defeq\,\,\, (\psi \supset z = t_1) \land (\lnot\psi \supset z = t_2).
	\] 
	 
	 Finally, for convenience, we often introduce formula and term abbreviations that are meant to expand as \( \L \)-formulas.  For example, we might introduce a new formula \( A \) by \( A \defeq \phi \), where \( \phi\in\L. \) Then any expression \( E(A) \) containing \( A \) is assumed to mean \( E(\phi). \) Analogously, if we introduce a new term \( t \) by  \( t = u \defeq \phi(u) \) then any expression \( E(t) \) is assumed to mean \( \exists u(E(u) \land \phi(u)). \)

Dynamic worlds are enabled by performing actions, and in the language, this is realized using the \emph{do} operator. That is, the result of doing an action \( a \) at situation \( s \) is the situation \( do(a,s). \)  Functional fluents, which take situations as arguments, may then have different values at different situations, thereby capturing \emph{changing properties} of the world. As noted, the constant \( \ins \) is assumed to give the actual initial state of the domain, but the agent may consider others possible that capture the beliefs and ignorance of the agent. In general, we say a situation is an \emph{initial} one when it is a situation without a \emph{predecessor}:  \[
	\init(s) \defeq \neg\thereis a, s'\!.~s=do(a,s').
\]
The picture that emerges is that situations can be structured as a set of trees, each rooted at an
initial situation and whose edges are actions. More conventions: we use $\ivar$ to range over such initial situations only, and  let $\alpha$ denote sequences of action terms or variables, and freely use this with $do$, that is, if $\alpha = [a\sub 1, \ldots, a\sub n]$ then $do(\alpha, s)$ stands for $do(a\sub n, do(\ldots, do(a\sub 1, s)\ldots ))$.

Domains are modeled in the situation calculus as \emph{axioms}. A set of \( \L \)-sentences specify the actions available, what they depend on, and the ways they affect the world.  Specifically, these axioms are given in the form of a  \emph{basic action theory} \cite{reiter2001knowledge}, reviewed below.

\subsection{Basic Action Theories} %
\label{sub:basic_action_theories}

In general, a basic action theory \( \D \) is a set of sentences consisting of (free variables understood as universally quantified from the outside): 
  
   \begin{itemize}
	\item an \emph{initial theory} \( \D_0 \) that describes what is true initially;

  \item \emph{precondition} \emph{axioms}, of the form $\poss(A(\vec x,s)) \equiv \beta(\vec x, s)$ that describe the conditions under which actions are executable;

\item \emph{successor state axioms}, of the form $f(\vec x, do(a,s)) = u \equiv \gamma\sub f (u, \vec x, s)$, that describe the changes to fluent values after doing actions; 

\item domain-agnostic \emph{foundational} \emph{axioms}, such as a second-order induction axiom for the space of situations and unique name axioms for actions, the details of which need not concern us here  \cite{reiter2001knowledge}.
  \end{itemize}

\noindent The formulation of successor state axioms, in particular, incorporates Reiter's monotonic solution to the frame problem \cite{Reiter:91}. 

An agent reasons about actions by means of entailments of a basic action theory \( \D \), for which standard first-order (Tarskian) models suffice (although see below). A fundamental task in reasoning about action is that of \emph{projection} \cite{reiter2001knowledge}, where we test which properties hold after actions. Formally, suppose $\phi$ is a situation-suppressed formula or uses the special symbol $\now.$  Given a sequence of actions $a\sub 1$ through $a\sub n$, we are often interested in asking whether $\phi$ holds after these: \[ \D \models \phi[do([a\sub 1, \ldots, a\sub n], \ins)]? \]

This concludes our review of the basic features of the language. In the subsequent sections, we will discuss how the formalism is first extended for knowledge, and then, degrees of belief against discrete probability distributions and beyond.  To prepare for that, the rest of the section will introduce three logical constructions that will be used in our work. First, we will define a class of Tarskian structures.   Second, we define a logical term standing for summation in the usual mathematical sense, that is to be understood as an abbreviation for a formula involving second-order quantification. Third, we analogously  define a logical term standing for integration in the usual mathematical  sense.

\subsection{$\real$-interpretations} 

For our purposes, the notion of entailment  will be assumed wrt a class of Tarskian structures that we call $\real$-interpretations. See \cite{enderton1972mathematical} for a review of Tarskian structures; we assume some familiarity with the underlying  notions. Below, for any $\L$-term $t$ and $\L$-interpretation $M$, we use $t^M$ to mean the domain element that $t$ references. If $t$ has a free variable $x$ and $y$ is any $\L$-term, then we write $t^x_y$ to mean the $\L$-term obtained by replacing $x$ in $t$ with $y$. 
Finally, for any variable map $\mu$ and variable $X$, we use $\mu^X_Y$ to mean a variable map that is exactly like $\mu$ except that for the  variable $X$ it takes the value $Y$.  

\begin{definition} By an \( \real \)-interpretation we mean any $\L$-structure  where  \(
=,\lt,\gt,0,1,+,\times,/,-,e,\pi \), exponentiation and logarithms have their usual interpretations.  
\end{definition}

(That is, ``$1+0 = 1$" is true in all $\real$-interpretations,  if ``$x \gt y$" is true  then ``$\neg (y \gt x)$" is true,  and so on.) So, henceforth, when we write $\Sigma \models \phi$, we mean that in all $\real$-interpretations where $\Sigma$ is true, so is $\phi$.\footnote{Alternatively, one could have  specified axioms for characterizing the field of real numbers together with $\Sigma$. Whether or not reals with exponentiation is {\it first-order\/} axiomatizable  remains a major open question \cite{Marker:1996fk}.}

Natural numbers can be defined in terms of a predicate by appealing to $\real$-interpretations.  Let 
\[ \natt(x) \doteq \forall P[ (P(0) \land \forall x (P(x) \supset P(x+1))) \supset P(x)].\]
Then: 

\begin{theorem}\label{thm:real interpretations natural domain} 
Let $M$ be any $\real$-interpretation, and $c$ a constant symbol of $\L$. Then, $M  \models\natt(c)$  iff $c\su M \in \nat$.  
\end{theorem}

The proof relies on two lemmas. 
First, we argue that the set of natural numbers satisfies the antecedent $A \doteq P(0) \land \forall x(P(x)\supset P(x+1))$: 

\blemma\label{lem:nat ante} 
For every $\real$-interpretation $M$ and map $\mu$, $M, \mu \su P\sub \nat \models A.$ 
\elemma 

\bprf 
Since ``0" and ``+" have their usual interpretations, and $0\in \nat$, $M, \mu\su P\sub \nat \models P(0).$ Suppose \( k\in \nat, \)  and so, $M, {\mu \su {P~ x} \sub {\nat~ k}} \models P(x)$. Since $\nat$ includes the successors of all its elements, $M, {\mu \su {P~ x} \sub {\nat~ k}} \models P(x+1).$ \qed 
\eprf 

Next, we argue that every set satisfying the antecedent includes the set of natural numbers:

\blemma\label{lem:nat subset} 
For every $M$ and $\mu$ as above, and for any set $Q$, if $M, \mu\su P\sub Q \models A$ then $\nat$ is a subset of $Q$. 
\elemma  

\bprf 
Suppose $M, \mu\su P\sub Q \models A$. We prove the claim by induction on natural numbers. For the base case, we consider $0\in \nat$, and by assumption $M, \mu\su P\sub Q \models P(0)$, and so $0$ is in the set $Q$. For the hypothesis, assume for any  $k \in\nat$,  \( k \) is also in $Q$. That is, $M, \mu\su {P~ x}\sub {Q~ k} \models P(x)$. Of course, the successor of $k$ is a natural number and by assumption, $M, \mu\su {P~ x}\sub {Q~ k} \models P(x+1)$. Thus, starting from 0, every natural number and its successor is in $Q$, and so $Q$ must include $\nat$. \qed
\eprf

So, the main claim is as follows: 

\bprf 
Suppose $M \models \natt(c)$, that is, $M \models \forall P(A\supset P(c)).$ Since $\nat$ is a subset of the domain by assumption, for any map $\mu$, $M, \mu\su P\sub \nat \models A\supset P(c)$. By Lemma \ref{lem:nat ante}, $M, \mu\su P\sub \nat\models A$, and so $M, \mu\su P\sub \nat \models P(c)$. Hence $c\su M\in \nat$.

Suppose instead $c\su M\in \nat$. Suppose for any set $Q$ and any map 
$\mu $, $M, \mu\su P\sub Q \models A$. By Lemma \ref{lem:nat subset}, $\nat$ is a subset of $Q$ and so $c\su M$ is in $Q$, that is, $M, \mu\su P\sub Q \models P(c)$. Therefore, $M, \mu\su P \sub \nat \models A\supset P(c)$. Since this holds for any set $Q$, $M \models \forall P[A\supset P(c)]$, that is, $M\models \natt(c)$. \qed 
\eprf

\subsection{Summation} %
\label{sub:summation}

Here we show how finite summations can be characterized as an abbreviation for an $\L$-term using second-order quantification. 

Let $f$ be any $\L$-function from $\mathbb{N}$ to $\real$. Let $\summ(f,n)$, standing for the sum of the values of $f$ for the argument $1$ through $n$, be defined as an abbreviation: 
\[ \begin{array}{l}
	\summ(f,n) = z \doteq \is g[g(1) = f(1) ~\land \\ \hspace{3cm}~~ g(n) = z ~\land \\ \hspace{3cm}~~ \forall i(~1\leq i\lt n \supset g(i+1) = g(i) + f(i+1)~)].
\end{array}  \]
The variable $i$ is understood to be chosen here not to
conflict with any of the variables in $n$ and $z$. The function $g$ from $\nat$ to $\real$ is assumed to not conflict with $f$. (That is, the logical terms  are distinct.) 

This can then be argued to correspond to summations in the usual mathematical sense as follows:

\bthm 
Let $f$ be a function symbol of $\L$ from $\nat$ to $\real$, $c$ be a term, and $n$ be  a constant symbol of $\L$. Let $M$ be any $\real$-interpretation. Then, \[ \textrm{if } \sum\sub {i=1}\su {n\su M} f\su M (i) = c\su M  \textrm{ then } M\models \summ(f,n) = c. \]
\ethm 

\bprf 
We prove by induction on $n\su M$. For the base case, suppose $n\su M =1$. Then, the antecedent would give us $f\su M(1)$. 
As for the consequent, clearly $M \models \summ(f,n) = f(1),$ and so the base case holds. 

Assume the hypothesis holds for $n\su M$, and we prove the case for $(n\su M + 1)$. 
(Note that owing to ``1" and ``+" having their usual interpretations, $(n+1)\su M = n\su M +1$.) So suppose \[ \sum \sub {i=1} \su {{n\su M} +1} f\su M (i) = c\su M.\]
On expanding the summation expression, we obtain: \[
	\sum\sub {i=1}\su {n\su M} f\su M (i) = c\su M - f\su M (n\su M + 1)
\]
By induction hypothesis, \( M \models \summ(f,n) =  c - f(n+1) \), and so, by definition, \( M\models \summ(f,n+1) = c. \) \qed 
\eprf

Henceforth, we write: \[ \disp \sum\sub {i=1} \su n t \] 
to mean the logical formula   $\summ(t, n)$ for a  logical term $t$. Here, $i$ is assumed to not conflict with any of the variables in $n$ and $t$.\footnote{That is, we use the standard mathematical notation to denote sums (and later: limits and integrals) in two ways that context will disambiguate: first, as the usual mathematical expression, and second, as a well-formed logical formula to be understood as an abbreviation as explained above.} 

It is worth noting that this logical formula can be applied to summation expressions where the arguments to the terms are not restricted to natural numbers, but taken from any finite set. For example, suppose $H$ is any finite set of terms $\{h_ 1, \ldots, h_ n\}$. We can then use terms such as: \[ \disp\sum _{h\in H} t(h)\]
standing for an abbreviation, similar to $\summ(t,n)$: let $g$ be a function, and let $g(i) = t(h_i)$. Then clearly the above sum defines the same number as $\sum_{i=1}^n g$.  In the sequel, we sometimes sum over a finite set of situations, or a finite vector of values, which is then understood as an abbreviation in this sense.

\subsection{Integration} %
\label{sub:integration}

Finally, we characterize integrals as logical terms.\footnote{In this article, given a non-negative real-valued function, our notion of an integral of this function  is based on the Riemann integral \cite{trench2003introduction}, 
in which case the function is said to be \textit{integrable}. There are limitations to the Riemann integral; for example, the function \( f\colon [0,1] \rightarrow \real \) where \[
	f(x) = 	\begin{cases}
	1 & \textrm{if } x \textrm{ is rational} \\
	0 & \textrm{otherwise}		
	\end{cases}
\]
is not integrable in the Riemann account. In the calculus community, generalizations to the Riemann integral, such as the gauge integral \cite{swartz2001introduction}, have been studied that allow for the integration of such functions. We have chosen to remain within the framework of classical integration, but other accounts may be useful. 

The key idea for moving beyond Riemann integrals would be to amend the logical formula standing for the abbreviations we introduce below. (For example, rather than partitioning the domain of a function, partitioning the range would allow us to formalise Lebesgue integrals.) 
} We present this first for a continuous real-valued function of \emph{one} variable, and then discuss its (straightforward) extension to the many-variable case. 

\newcommand\LIM[2]{\textsc{LIM}[#1,#2]}

We begin by introducing a notation for limits to positive infinity. For
any logical term $t$ and variable $x$, we introduce a term characterized as follows:
\[
\LIM{x}{t} = z \,\,\defeq\,\,
\forall{u}(u>0 \supset \exists{m}\,\forall{n}(n>m \supset \abs{z-t^x_n}<u)).
\]
The variables $u$, $m$, and $n$ are understood to be chosen here not to
conflict with any of the variables in $x$, $t$, and $z$. The abbreviation can be argued to correspond to the limit of a function at infinity in the usual sense: 

\begin{lemma}\label{lem:limits}
Let $g$ be a function symbol of $\L$ standing for a function from $\real$ to
$\real$, and let $c$ be a constant symbol of $\L$. Let $M$ be
any $\real$-interpretation of $\L$. Then we have the following:
\[
\mbox{If}\,
\lim_{x\to\infinity} g\su M (x) \,=\, c^M
\,\,\,\,\mbox{then}\,\,\,
M \,\models\, (c\,=\, \LIM{x}{g}).
\]
\end{lemma}

\begin{proof} Suppose the limit holds. Then, by the standard notion of limits \cite{keisler2012elementary}, for every real number \( \epsilon \gt 0, \) there is a natural number \( j \) such that for all natural numbers \( k \gt j \): \[
	| g\su M (k) - c\su M | \lt \epsilon. 
\]
Suppose \( M \not\models \LIM{x}{g} = c. \)	Then there is some \( r \gt 0 \) such that \( M, \mu \su u \sub r \not\models \exists m \forall n (n\gt m \supset | c - g(n)| \lt u). \) By Theorem \ref{thm:real interpretations natural domain}, the domain of \( M \) includes \( \nat, \) and so 
let \( m \) and \( n \) be variables that are mapped to \( j \) and \( k \) respectively. Then, given that \( M \) interprets arithmetic symbols in the usual way, the claim \( M, \mu \su {u~m~n} \sub {r~j~k} \not\models (n\gt m) \supset (| c - g(n) | \lt u) \) is a contradiction. \qed

\end{proof}

Henceforth, we write:\[
	\lim_{x\to\infinity}{t}
\]
to mean the logical formula  \( \LIM{x}{t}. \) 

\newcommand\integ[4]{\textsc{INT}[#1,#2,#3,#4]}

Next, for any variable $x$ and terms $a$, $b$, and $t$,
we introduce a term $\integ{x}{a}{b}{t}$ denoting the definite integral of
$t$ over $x$ from $a$ to $b$:
\[
\integ{x}{a}{b}{t} \,\,\defeq\,\,
\lim_{n\to\infinity}\,h\cdot \sum_{i=1}^n\, t^x_{(a+h\cdot{i})} 
\]
where
$h$ stands for
\((b-a)/n\). 
The variables are chosen not to conflict with any of the
other variables. We now show: 

\begin{lemma}\label{lem:definite integral}  Let $g$ be a function symbol of $\L$ standing for a function from $\real$ to
$\real$, and let $a, b, c$ be  constant symbols of $\L$. Let $M$ be
any $\real$-interpretation of $\L$.
Then we have the following:
\[
\mbox{If}\,
\int_{a\su M}^{b\su M}\! g^M(x)\,dx \,=\, c^M
\,\,\,\,\mbox{then}\,\,\,
M \,\models\, (c=\integ{x}{a}{b}{g}).
\]
\end{lemma}

\begin{proof} Suppose $\int\sub {a\su M} \su {b\su M} g\su M (x) dx = c\su M$, that is, $g\su M$ is integrable. By definition, then: \[
	\lim\sub{k\to\infinity} h\cdot \sum\sub {i=1} \su {k} g\su M({a\su M} + h\cdot i)  = c\su M
\]
where $h = (b\su M - a\su M)/k$. 	
  By Lemma \ref{lem:limits}, we have \[ M \models \lim\sub {n\to \infinity} \,h\cdot \sum_{i=1}^n\, g^x_{(a+h\cdot{i})} = c \]
where \( h = (b-a)/n \), that is, \( M \models \integ{x}{a}{b}{g} = c. \) \qed 
\end{proof}

Finally, we define the definite integral of $t$ over all real values of $x$ by
the following:
\[
\int_{x}\!{t} \,\,\defeq\,\,
\lim_{u\to\infinity} \, \lim_{v\to\infinity} \, \integ{x}{-u}{v}{t}.
\]
The main result for this logical abbreviation is the following:

\begin{theorem}\label{thm:singleint}
Let $g$ be a function symbol of $\L$ standing for a function from $\real$ to
$\real$, and let $c$ be a constant symbol of $\L$. Let $M$ be
any $\real$-interpretation of $\L$.
 Then we have the following:
\[
\mbox{If}\,
\int_{-\infinity}^{\infinity}\! g^M(x)\,dx \,=\, c^M
\,\,\,\,\mbox{then}\,\,\,
M \,\models\, (c=\!\!\int_{x}\!{g(x)}).
\]

\end{theorem}

\begin{proof} Suppose $\int\sub {-\infinity} \su \infinity g\su M (x) dx = c\su M$, that is, $g\su M$ is integrable. By definition, \[
	\lim\sub{u\to\infinity} \lim\sub{v\to\infinity} \int\sub{-u}\su {v} g\su M (x) = c.
\]
By Lemma \ref{lem:limits} and Lemma \ref{lem:definite integral}, we obtain:\[
	M \models \lim\sub{u\to\infinity} \lim\sub{v\to\infinity} \integ{x}{-u}{v}{g} = c.
\]
That is, \( M\models \int \sub x g(x) = c.  \) \qed  
	
\end{proof}

\newcommand{\minteg}{\mbox{MINT}}

The characterization of integrals for a many-variable function $f$, from $\real\su k$ to $\real$, is then an easy exercise. For variables $x\sub 1, \ldots, x\sub k$ and terms $a\sub 1, \ldots, a\sub k$, $b\sub 1, \ldots, b\sub k$, and $t$,
we introduce a term $\minteg[x\sub 1, \ldots, x\sub k, a\sub 1, \ldots, a\sub k, b\sub 1, \ldots, b\sub k, t]$ denoting the definite integral of
$t$ from $(a\sub 1, \ldots, a\sub k) \in \real\su k$ to $(b\sub 1, \ldots, b\sub k)\in \real\su k$:
\[
\minteg[x\sub 1, \ldots, x\sub k, a\sub 1, \ldots, a\sub k, b\sub 1, \ldots, b\sub k, t] \,\,\defeq\,\,
\lim_{n\sub 1\to\infinity}\ldots   \lim_{n\sub k\to\infinity}\,h\sub 1  \cdots  h\sub k \cdot  \sum_{i{\sub 1}=1}^{n\sub 1}\,\ldots  \sum_{{i\sub k}=1}^{n\sub k}\, t^{x\sub 1, \ldots, x\sub k}_{(a\sub 1+h\sub 1 \cdot{i\sub 1}), \ldots, (a\sub k + h\sub k\cdot {i\sub k})} 
\]
where
$h\sub j$ stands for
\((b\sub j-a\sub j)/{n\sub j}\).  The variables are chosen not to conflict with any of the
other variables. Finally, we define the definite integral of $t$ over all real values of $x\sub 1, \ldots, x\sub k$ by
the following:
\[
\int_{x\sub 1, \ldots, x\sub k}\!{t} \,\,\defeq\,\,
\lim_{u\sub 1\to\infinity}\ldots \lim_{u\sub k\to\infinity} \, \lim_{v\sub 1\to\infinity}\ldots \lim_{v\sub k\to\infinity} \, \minteg[x\sub 1, \ldots, x\sub k, -u\sub 1, \ldots, -u\sub k, v\sub 1, \ldots, v\sub k, t].
\]
From this we get, as a corollary to Theorem \ref{thm:singleint}:

\begin{corollary}
Let $g$ be a function symbol of $\L$ standing for a function from $\real\su k$ to
$\real$, and let $c$ be a constant symbol of $\L$. Let $M$ be
any $\real$-interpretation of $\L$.
Then we have the following:
\[
\mbox{If}\,
\int_{-\infinity}^{\infinity}\ldots\int_{-\infinity}^{\infinity}\! g^M(x\sub 1, \ldots, x\sub k)\,d x\sub 1\ldots d x\sub k \,=\, c^M
\,\,\,\,\mbox{then}\,\,\,
M \,\models\, (c=\!\!\int_{x\sub 1, \ldots, x\sub k}\!{g(x\sub 1, \ldots, x\sub k)}).
\]

\end{corollary}

\section{A Theory of Knowledge} %
\label{sub:knowledge}

\subsection{Knowledge} 
An early treatment of \emph{knowledge} in the situation calculus is due to Moore \cite{moore1985}. The classical \emph{possible-world} interpretation for knowledge \cite{Kripke:1963uq,nla.cat-vn2250548} is based on the notion that there many different ways the world can be, where each world stands for a complete state of affairs. Some of these are considered possible by a putative agent, and they determine what the agent knows and does not know. Moore's observation was that situations 
 can be viewed as possible worlds. ({Different from standard modal logics \cite{reasoning:about:knowledge}, however, worlds are reified as part of the syntax, but this is a minor technicality.}) A special binary {fluent}  \( K \), taking two situation arguments, determines the accessibility relation between worlds: \( K(s',s) \) says that when the agent is at \( s \), he considers \( s' \) possible. ({Also different from standard modal logics, we note that the order of the terms in the accessibility relation is reversed.}) 
\emph{Knowledge}, then, is simply truth at accessible worlds:

 \begin{ldefn} \textbf{(Knowledge.)} Let \( \phi \) be any situation-suppressed formula. 
The agent \emph{knowing} \( \phi \) at situation \( s \), written $\textit{Knows}(\phi,s)$, is the following abbreviation: 
	\[ 	\mathit{Knows}(\phi,s) \defeq \forall s'.~ K(s',s) \supset \phi[s']. \]
\end{ldefn}

\noindent In English: if in every situation \( s' \) that is considered possible at \( s \) the formula \( \phi \) holds, then the agent {knows} \( \phi \) at \( s. \)\footnote{We do not insist that beliefs are necessarily true in the real world. Perhaps for this reason, ``{believes}'' is a more appropriate reading of \( \know \). The two terms are used  interchangeably here. Also, see Scherl and Levesque \cite{citeulike:528170} on how features such as \emph{positive} and \emph{negative} \emph{introspection} can be enabled by constraining the accessibility relation \( K \), in a manner entirely analogous to standard modal logic \cite{reasoning:about:knowledge}. } 

Moore's account has been adapted to the arrangement of dynamic laws via basic action theories by Scherl and Levesque \cite{citeulike:528170}. In this scheme, the initial theory is assumed to specify the agent's initial beliefs. For example, $\textit{Knows}(\textit{velocity}(\textit{obj5}) = 50, \ins)$ expands to $\forall s'.~K(s',\ins) \supset \textit{velocity(obj5,}s') = 50;$ that is, every accessible world initially agrees that the velocity of object 5 is 50, which means that the agent knows that the velocity of object 5 is 50. In addition, 
$\neg \textit{Knows}(\textit{velocity}(\textit{obj6}) =60, \ins)$ and $\textit{Knows}(f=0 \lor f=1, \ins)$ say that initially, the agent does not know that the velocity of object 6 is 60, and knows that the functional fluent $f$ takes a value of either 0 or 1. Equivalently, one specifies an initial constraint on $K$; for example:
\beq\label{eq:categorical knowledge}
 K(\ivar, \ins) \equiv (f = 1 \lor f =0)[\ivar] 
\eeq 
characterizes the agent's knowledge about the fluent $f$ taking a value of either 0 or 1.

\subsection{Sensing} 

To specify the behavior of \( K \) at non-initial situations, Scherl and Levesque provide a \emph{successor state axiom} for \( K \). This axiom, intuitively, tests whether situations are to remain accessible as actions occur.  Without going into full details, assume the provision of domain-specific \emph{sensing} axioms, also part of the basic action theory \( \D \). For example, \[
	\ti{SF}(\ti{sensetrue-f},s) \equiv f(s) = 1.
\] 
formalizes the sensing outcome for an action to check  whether \( f \) has value $1$  in situation \( s. \) Here \( \ti{SF} \) is a special \( \L \)-predicate, similar to \( \poss. \) Then, for \( s' \) to be accessible from \( s \), we need the following successor state axiom to be included in \( \D \):  
\[
	K(s', do(a,s)) \equiv \exists s''[K(s'', s) \land s' = do(a,s'')  \land Poss(a, s'') \land (\ti{SF}(a,s'') \equiv \ti{SF}(a,s))]. 
\]
This says that if \( s'' \) is the predecessor of \( s', \) such that \( s'' \) was considered possible at \( s \), then   \( s' \) would be considered possible from \( do(a,s) \) contingent on sensing outcomes. 

To appreciate how this axiom works in dynamical domains, assume that for actions without any sensing aspect, such as  an action that toggles the value of \( f, \) one simply lets $\ti{SF}$ be vacuously true: \[
	\ti{SF}(\ti{toggle-f},s) \equiv \true. 
\]
The idea, then, is that the accessibility between situations would not change as physical actions occur. However, as the agent operates in an  environment where it senses various properties, those situations that are \emph{incompatible} with the real world regarding the sensed value will be deemed \emph{impossible} after such sensing actions. This leads to a notion of  knowledge \emph{expansion},\footnote{Revising beliefs, where the agent believes \( \phi \) but acquires information to now believe \( \neg \phi \), is not dealt with in this work. See \cite{DBLP:journals/ai/ShapiroPLL11} for an account.} as the agent becomes more certain about the true nature of the world. See Figure \ref{fig:kssa} for an illustration of the axiom: we imagine three situations \( s, s'  \) and \( s'' \) that are epistemically related prior to any sensing (that is, \( s' \) and \( s'' \) are possible worlds when the agent is at \( s \)). These situations disagree on the value of \( f \), and consequently, the agent does \textit{not} know \( f \)'s value. After executing a sensing action for the truth of \( f \), however, \( do(\textrm{\it sensetrue-f}, s'') \) is not epistemically related to \( do(\textrm{\it sensetrue-f}, s) \). The upshot is that the agent knows the value of \( f \) at \( do(\textrm{\it sensetrue-f},s) \) and believes that this value is 1.

\begin{figure}[h]
  \centering
    \includegraphics[width=.4\textwidth]{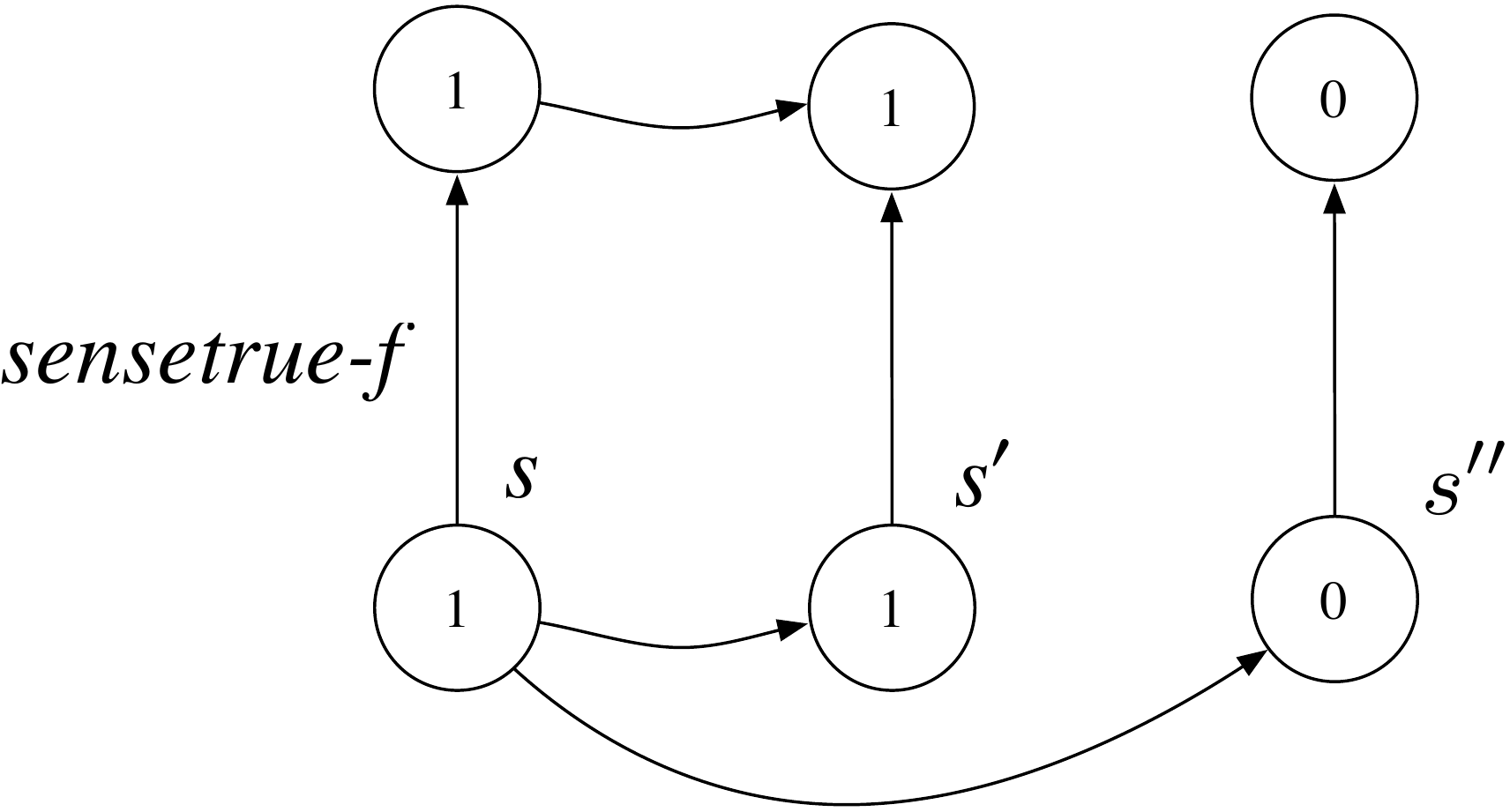}
  \caption{Situations with accessibility relations after sensing. The numbers inside the circles denote  the \( f \) values at these situations.}
  \label{fig:kssa}
\end{figure}

\subsection{Degrees of Belief and Likelihood} %
\label{sub:belief_and_likelihood}

The Scherl and Levesque scheme, however, lacks constructs to quantify the agent's uncertainty. One measure to quantify uncertainty is with \emph{degrees of belief.} What is also lacking in their scheme is the ability to formalize the \emph{probabilistic noise} in  effectors and sensors, as seen in many real-world robotics  applications \cite{thrun2005probabilistic}. These limitations, to discrete approximations, were addressed by BHL \cite{Bacchus1999171}. 

The BHL scheme builds on Scherl and Levesque's ideas, especially regarding how accessibility relations between worlds vary as a result of actions. In fact, the reader may observe many parallels between the two extensions. BHL's remarkably simple proposal consists of introducing two new distinguished fluents, \( p \)  and \( l \), in addition to \( K \). We present a simpler alternative involving only \( p  \) and 	\( l. \) 

As a simple running example, imagine a robot moving towards a wall, as shown in Figure \ref{fig:robot}. Its distance to the wall is given by a functional fluent \( h, \) and it is assumed to be equipped with a sonar sensor that measures how far the robot is from the wall. In other words, ideally, the sensor's reading would correspond to the actual value of \( h. \)

The \( p \) fluent determines a probability distribution on situations, by associating situations with \emph{weights}. More precisely, the term $p(s'\!,s)$ denotes the
relative \emph{weight} accorded to situation \( s' \) when the agent happens
to be in
situation \( s. \) Of course, \( p \) can be seen as a companion to \( K \). As one would for \( K, \)
the properties of $p$ in initial states, 
which vary from domain to domain, are specified with  axioms as part of \(
\D_0. \) For example, \beq\label{eq:uniform-1} \begin{array}{l}
	p(\ivar,\ins) =	 u \equiv  ((\xpos(\ivar) = 2 \lor \xpos(\ivar) = 3) \land u = .5) ~\lor \\ \qquad \mbox{} \qquad \mbox{} \qquad \mbox{} (\xpos(\ivar) \neq 2 \land \xpos(\ivar) \neq 3 \land u=0 ).
\end{array}
\eeq 
says that those initial situations where \( \xpos \) has the integer values 2 or 3 obtain a weight of .5. All other situations, then,  obtain 0 weight. We expect, of course, that weights are nonnegative, and that  non-initial situations are given a weight of 0 initially.
The following nonnegative constraint, also part of \( \D_0 \), ensures this:
\begin{equation}\label{eq:p initial constraint}
\all \ivar,s.\, p(s,\ivar) \geq 0  \,\land\, (p(s,\ivar) \gt 0 \supset
\init(s)) 
\tag{P1}
\end{equation}
Note that this is a stipulation about initial situations $\ivar$ only. But BHL provide
a \emph{successor state axiom} for \( p \), to be listed   shortly, that ensures that this constraint holds in all situations. 

Next, the term $l(a,s)$ is intended to denote the likelihood of action $a$ occurring in
situation $s$. Among other things, \( l \) can be used to model noisy sensors. 
 This is perhaps best demonstrated using an example. Imagine a sonar aimed at the wall, which gives a reading for the true value of \( \xpos. \) Supposing the sonar's readings are subject to additive Gaussian noise.\footnote{Note that  Gaussians  are continuous distributions involving
  $\pi,$ $e,$ exponentiation, and so on. Therefore, BHL always consider
  discrete probability distributions that \emph{approximate} the continuous
  ones.} If now a reading of \( z \) were observed on the sonar, intuitively,  those situations where \( \xpos = z \) should be considered more probable than those where  \( \xpos \neq z \).\footnote{As usual, the reading observed is not in the control of the agent. Here, we assume that the value is given to us, and in that sense, the language is geared for projection (cf. Section \ref{sub:example_first}).  For example, we might be interested in the beliefs of the agent after obtaining a specific sequence of readings on the sonar. Integrating this language with an \textit{online}  framework that obtains such readings from an external source is addressed in \cite{Belle:2015ab}.} This occurrence is captured using likelihoods in the formalism. Basically, if \( \sonar(z) \) is the sonar sensing action with \( z \) being the value read, we  specify a likelihood axiom describing its error profile as follows: 
\begin{equation}\label{eq:sonar error profile simple} 
\begin{array}{l}
l(\xsense(z),s) = u \equiv  (z\ge0 \land u=\N(z-\xpos(s);\mu,\sigma^2)) 
~\lor \\ \qquad \mbox{} \qquad \mbox{} \qquad \mbox{} \qquad \mbox{}  (z<0 \land u=0). 
\end{array}
\end{equation}
\noindent  This stipulates that the difference between a nonnegative reading
of \( z \) and the true value \( \xpos \) is normally distributed with a
variance of \( \sigma^2 \) and mean of \( \mu \).\footnote{We understand $\N(u;\mu, \sigma^2)$ as an abbreviation for the mathematical expression $e^{\frac{-(u-\mu)^2}{2\cdot \sigma^2}} /  \sqrt{2 \cdot \pi\cdot \sigma^2}$.}

 Clearly, the error profile of various hardware devices is application dependent, and it is this profile that is modeled as shown above using \( l. \) Notice, for example, when $\mu = 0$, which indicates that the sensor has no systematic bias, then \( l(\sonar(5),s) \) will be higher when \( \xpos(s) = 5 \) than when \( \xpos(s) = 25. \) Roughly, then, the idea is that after an observation, the weights on situations would get redistributed based on their compatibility with the observed value.

One may contrast such likelihood specifications to (trivial) ones for {deterministic} physical actions,\footnote{Noisy actions  will also involve non-trivial likelihood axioms.  Their treatment, however,  is deferred to a subsequent section.} such as an action \emph{move(z)} of moving towards the wall by precisely \( z \) units. For such actions, we  simply write \[
	l(\ti{move}(z),s) = u \equiv u = 1 
\]
in which case the \( p \) value of \( s \) is the same as that for \( do(\ti{move(3)},s) \). Thus, this is a form of \emph{imaging} \cite{lewis1976probabilities}, where the weights of worlds are simply ``transferred" to their successors. \smallskip

Formally, we add \emph{action likelihood} axioms to \( \D \): 

\begin{definition} Action likelihood axioms for each action type \( A \) are sentences of the form: \[
	l(A(\vec x),s) = u \equiv \psi_A(\vec x, u, s).
\]
Here \( \psi_A(\vec x, u, s) \) is any formula characterizing the conditions under which action \( A(\vec x) \) has likelihood \( u \) in \( s \).
\end{definition}

In general,   likelihood axioms can  depend on any number of  features of the world besides the fluent that the sensor is measuring. For example, imagine that the sonar's accuracy depends on the room temperature. We could then specify an error profile as follows: \begin{equation}\label{eq:sonar error profile complicated}
	\begin{array}{l}
		l(\xsense(z), s) = u \equiv \\
	\qquad \mbox{} \qquad \mbox{} \qquad \mbox{} 	(z\geq 0 \land \mathit{temp}(s) \geq  0 \land u = \N(z - \xpos(s);\mu, 1)) ~\lor \\
	\qquad \mbox{} \qquad \mbox{} \qquad \mbox{} 	(z \geq 0 \land \mathit{temp}(s) < 0 \land u = \N(z-\xpos(s);\mu,16)) ~\lor \\
	\qquad \mbox{} \qquad \mbox{} \qquad \mbox{} 	(z <0 \land u = 0). 
	\end{array}
\end{equation}
That is, the sonar's accuracy worsens severely when the temperature drops below 0, as seen by the larger variance.  \smallskip 

Having introduced the new fluents, we are now ready to provide the successor state axiom for \( p, \) which is  analogous to the one for \( K \): \begin{equation}\label{eq:p ssa}
	\begin{array}{l}
		p(s',do(a,s)) = u \,\,\equiv\\
		\quad \quad \exists s''~[s' = do(a,s'') \land \poss(a,s'') \,\land\\
         \quad\quad\quad\quad\quad u = p(s'',s) \times l(a,s'')] \\
		\quad \lor\,\, \lnot\exists s''~[s' = do(a,s'') \land \poss(a,s'')]  
           \land u = 0. 
	\end{array}
	\tag{P2}
\end{equation}

This says that  the weight of situations \( s' \) relative to \( do(a,s) \)
is the weight of their predecessors \( s'' \) times the likelihood of \( a \)
contingent on the successful execution of \( a \) at \( s''. \)

\begin{figure}[h]
  \centering
    \includegraphics[width=.4\textwidth]{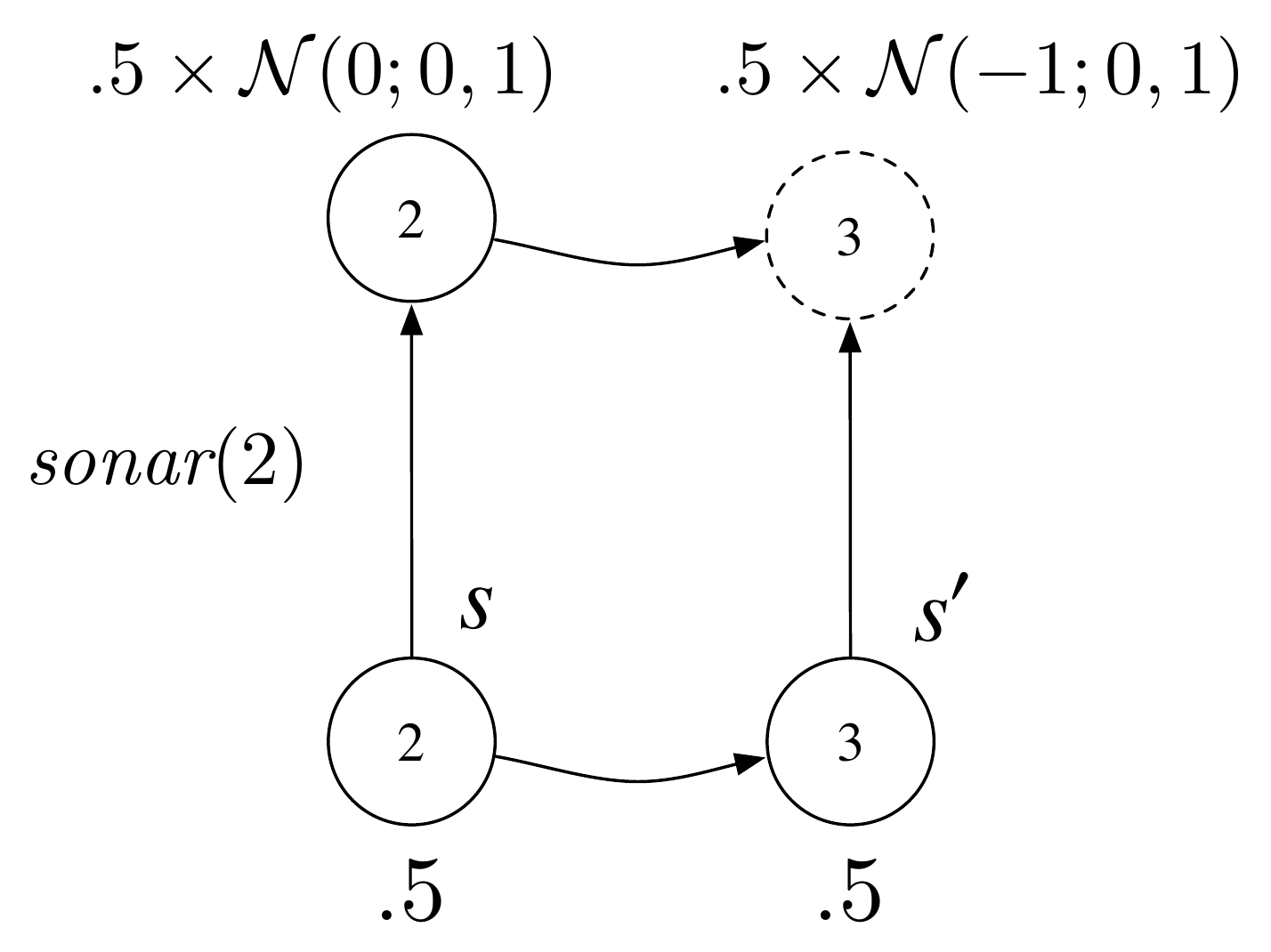}
  \caption{Situations with accessibility relations after noisy sensing. The numbers inside the circles denote the \( h \) values at these situations. Dotted circles denote a lower weight in relation to their epistemic alternatives.}
  \label{fig:sensing}
\end{figure}

To see an application of this axiom using the specifications \eqref{eq:uniform-1} and \eqref{eq:sonar error profile simple}, consider two situations \( s \) and \( s' \) associated with the same weight initially, as shown in Figure \ref{fig:sensing}. These situations have \( h \) values \( 2 \) and \( 3 \) respectively. 
 Suppose the robot obtains a reading of 2 on the sensor. Given a sensor with mean \( \mu = 0 \) and variance \( \sigma\su 2 = 1, \) the likelihood axiom is such that the weight of the successor of \( s \) is higher than that of \( s' \) because the \( h \) value at \( s \) coincides with the sensor reading. More precisely, the weight for the successor of \( s \) is given by the prior weight .5 multiplied by the likelihood factor \( \N(z-h(s);\mu,\sigma\su 2) = \N(2-2;0,1) = \N(0,0,1). \) The weight for the successor of \( s' \) is obtained analogously.

Interestingly, by means of the above axiom, if predecessors are not \emph{epistemically} related, which is another way of saying that \( p(s'', s) \) is 0, then their successors will also not be. Similarly, when \( a \) is not executable at \( s'', \) its successor is no longer accessible from \( do(a,s). \) One other  consequence of \eqref{eq:p initial constraint} and \eqref{eq:p ssa} is that \(
(p(s'\!,s) > 0) \) will be true only when $s'$ and $s$ share the same history
of actions. This is because \eqref{eq:p initial constraint} insists that initial situations are only epistemically related to other initial ones, and \eqref{eq:p ssa} respects this relation over actions.

We are now prepared to define the \emph{degree of belief} in a formula \( \phi \) at a situation \( s \), written \( \bel(\phi,s). \)  ({Henceforth, whenever a formula \( \phi \) appears in the context of \( \bel, \) we assume that it is either situation-suppressed,  or it only mentions the situation term \emph{now}.}) Intuitively, this is simply the \emph{weight} of \emph{accessible situations}. Formally:

\begin{definition}\label{defn:belief} \textbf{(Degrees of belief.)} Suppose \( \phi \) is any situation-suppressed  \( \L \)-formula. Then the \emph{degree of belief} in \( \phi \) is an abbreviation for: \[
\bel(\phi,s) \,\defeq\,
  \frac{1}{\gamma}\!\sum_{\{s':\phi[s'\!]\}} p(s'\!,s), 
\]
where \( \gamma \), the normalization factor, is understood (throughout) as the
same expression as the numerator but with \( \phi \) replaced by \( \true. \)
So, here, \( \gamma \) is \( \sum_{s'} p(s'\!,s). \) 
\end{definition}

\noindent Note that we do not have to insist that \( s' \) and \( s \) share histories
since $p(s'\!,s)$ will be $0$ otherwise, as discussed above. The summation term in this logical formula is not a new logical symbol, but simply an abbreviation for a second-order formula, by way of Section \ref{sub:summation}.  

\subsection{Discussion} %
\label{sub:discussion}

Let us conclude this section by remarking that since $p$ is a fluent, the syntax of basic action theories allows us to express probabilistic knowledge in a very general way, quite beyond standard probabilistic formalisms \cite{books/daglib/0023091,pearl1988probabilistic}. For example, to represent categorical uncertainty like \eqref{eq:categorical knowledge}, we would do:
\beq\label{eq:uniform-3} 
  p(\ivar,\ins) = \begin{cases}
 1 & \textrm{if } (h=2\lor h=3)[\ivar] \\
 0 & \textrm{otherwise}
 \end{cases}
 \eeq 
 In English: all initial situations where the value of $h$ is either 2 or 3 are considered epistemically possible, and accorded a weight of 1. All other initial situations are accorded a weight of 0.  
 Observe, however, that this $p$-specification does not say which value of $h$ is more likely. Thus, unlike standard probabilistic frameworks, we do not assume that it is always possible to find a single probability distribution for the robot to use. 
  
  It is also possible to handle partial specifications. Let us contrast \eqref{eq:uniform-1} with the following initial axiom for \( p \): \beq\label{eq:uniform-2} 
  p(\ivar,\ins) = \begin{cases} 
  .1 & \textrm{if } h(\ivar) \in \set{1,2,\ldots,10} \\
  0 & \textrm{otherwise} 
  \end{cases}
  \eeq
  Then, letting a basic action theory include the sentence: \[ \eqref{eq:uniform-1} \lor \eqref{eq:uniform-2} \] 
  means that the robot believes $h$ is uniformly distributed on $\set{2,3}$ or on $\set{1,\ldots,10}$ without being able to say which. To reiterate, we do not assume that it is always possible to find a single probability distribution for the robot to use.

  Of course, a much weaker specification is possible by replacing one of these probabilistic alternatives with categorical ones. For example, suppose the basic action theory were to instead include:\[ \eqref{eq:uniform-3} \lor \eqref{eq:uniform-2}. \] 
 In this case, the agent believes that $h$ may take a value from $\set{2,3}$ or is uniformly distributed on $\set{1,2,\ldots,10}$ without being able to say which. 
 
 In sum, the framework allows one to freely combine categorical and probabilistic specifications, leading to a very general model of belief. For simplicity of presentation, we consider a fairly simple set of specifications in this article, and refer readers to \cite{DBLP:journals/japll/BelleL15} for more involved ones.

\section{Belief Reformulated} %
\label{sec:degree_of_belief_reformulated}

The definition for degrees of belief, as given by BHL, is intuitive and simple. It is closely fashioned after the semantics for  belief in modal probability logics \cite{174658}, where the probabilities of formulas is calculated from the weights of possible worlds satisfying the formula. 
 Unfortunately, this definition is not easily amenable to   generalizations. 
Notice, for example, that \( \bel \) is well-defined only 
when the sum over all 
those situations \( s' \)
such that \( \phi[s'] \) holds is finite. This immediately  
precludes domains that involve an infinite set of situations agreeing on a formula. Moreover, the definition does not have an obvious analogue for 
continuous probability distributions. 
Observe that such an analogue would involve \emph{integrating} over the space of situations, which makes little sense. Indeed, it is not  certain what the space of situations would look like in general, but even if this was fixed, how such a thing can be tinkered with so as to obtain an appropriate notion of \emph{integration} is far from obvious. 

 Therefore, what we propose is to shift the calculating of probabilities from situations to \emph{fluent values}, that is, to the well-understood  domain of  numbers. The current section is an exploration of this idea. What we will show in this section is that Definition \ref{defn:belief} can be reformulated as a summation over numeric indices. That will allow, among other things, a seamless generalization from summation to integration, which is to be the topic of the next section. 

To prepare for that, in addition to the usual case notation used, for example, in \eqref{eq:uniform-2}, we will introduce another kind of conditional term for convenience. 
This involves a quantifier and a default value of $0$, like in
formula \eqref{eq:p ssa}. 
If \( z \) is a variable,
\( \psi \) is a formula and \( t \) is a term, we use \(
\probabbr{z}{\psi}{t} \)
as a logical term characterized as follows: 
\begin{eqnarray*}
\lefteqn{\probabbr{z}{\psi}{t} = u \,\,\defeq}\\[-\smallskipamount]
& & [(\exists
  z \psi) \supset \forall z(\psi \supset u=t)] 
  \land [(\neg\exists z
  \psi) \supset u=0)].
\end{eqnarray*}

\noindent The notation says that when \( \thereis z\psi \) is true, the value
of the term is $t$; otherwise, the value is $0$.  
When $t$
uses $z$ (the usual case), this will be most useful if there is a {\it
 unique\/} $z$ that satisfies $\psi.$

Returning to the task at hand, we will now need a way to enumerate the primitive fluent terms of the language. Intuitively, these correspond to the probabilistic variables in the language. Perhaps the simplest way is to assume there are $n$ fluents
$f_1, f_2,\ldots, f_n$ in \(\L\) which take no arguments other than
the situation argument,\footnote{Essentially, functional fluents in \( \L \) are  assumed to not take any object arguments. More generally, if we assume that the arguments of \( k \)-ary fluents are drawn from finite sets, an analogous enumeration of ground functional fluent terms is possible. %

Understandably, from the point of view of situation calculus basic action theories, where fluents are also usually allowed to
take {arguments} from \emph{any} set, including infinite ones, this is a limitation. But in probabilistic
terms, this would correspond to having a joint probability distribution over
infinitely many, perhaps uncountably many, random variables. We know of no
general logical account of this sort, and we have as yet no good ideas about how to deal
with it. 
It remains to be seen whether ideas from probability theory on high dimensions \cite{Billingsley95,da2006introduction} and infinite-dimensional probabilistic graphical models  \cite{DBLP:conf/uai/SinglaD07} can be leveraged for our purposes. 
}  
and that they take their values from some finite sets.  We can then
 rephrase  Definition \ref{defn:belief} as follows: 

\begin{definition}\label{defn:bel with two sums} Suppose \( \phi \) is as before. Let \( \bel(\phi,s) \) be an abbreviation for: \[
 \frac{1}{\gamma}
  \sum_{\vec{x}} \sum_{s'}
    \left\{
    \begin{array}{ll}
      p(s'\!,s) & \mbox{if\, ${\bigwedge f_i(s')=x_i \,\land\, 
         \phi[s'\!]}$}\\ 0
      & \mbox{otherwise}
    \end{array}
    \right.
\]
where \( \gamma \) is the numerator but with \( \phi \) replaced by \( \true \), as usual. 
\end{definition}

(For readability, we often drop the index variables in sums and connectives when the context makes it clear: in this case, $i$ ranges over the set $\{1, \ldots, n\}$, that is, the indices of the fluents in $\L$.)  
Definition \ref{defn:bel with two sums} suggests that for each possible value of the fluents, we are to sum over all possible
situations and for each one, if the fluents have those values and $\phi$
holds, then  the $p$ value is to be used, and 0 otherwise. Roughly speaking, if one were to group situations satisfying \( \bigwedge f_i(s) = x_i \) into sets for every possible vector \( \vec x, \) the union of these sets would give the space of situations. Our claim about the relationship between the two abbreviations can be made precise as follows: 

\begin{theorem}\label{thm:belief two sums} Let \( \D \) be any basic action theory and \( \phi \) any \( \L \)-formula. Then the abbreviations for \( \bel(\phi,s) \) from Definition \ref{defn:belief} and \ref{defn:bel with two sums} define the same number. 
	
\end{theorem}

\begin{proof} For the proof, we focus solely on the numerators of the two abbreviations. That is, \begin{equation}\label{eq:N1T1}
	\sum_{\{s':\phi[s'\!]\}} p(s'\!,s)
	\tag{$\dag$}
\end{equation}
on the one hand, and \begin{equation}\label{eq:N2T1}
	\sum_{\vec{x}} \sum_{s'}
	    \left\{
	    \begin{array}{ll}
	      p(s'\!,s) & \mbox{if\, ${\bigwedge f_i(s')=x_i \,\land\, 
	         \phi[s'\!]}$}\\ 0
	      & \mbox{otherwise}
	\end{array}
    \right. 
\tag{$\ddag$}
\end{equation}
on the other. We show that these expressions define the same number. With this, the case for the denominators  follows trivially, since \( \true \) is a special case of \( \phi \). 
Then, the claim is proven.

	Let  \( S  \) be a set  such that  \( {s'} \in S \) iff \( p(s', s) >0 \). (That is, for any ground situation term $s$, this is the set of all ground situation terms $s'$ such that $p(s',s)>0$.) 
	Then, let \( T \subseteq S \) be the set such that \( {s'} \in T \) iff \(  \phi[s'] \).  Intuitively,  \( T \) is the set of  all situations that are epistemically related to  \( s \), but where \( \phi \) holds. 
	 It is easy to  see that \[
		 \eqref{eq:N1T1} =  \sum_{s' \in T} p(s',s). 
	\] 
	
\noindent Suppose now $f\sub i$ ranges over $\set{c\sub {1i}, \ldots, c\sub {ki}} $, and so by extension, suppose 
$\vec f = \langle f\sub 1, \ldots, f\sub n \ra$   ranges over \(  \{ {\vec {c_1 {}}}, {\vec  {c_2 {}}}, \ldots, {\vec {c_k {}}} \}.  \) 
For any 
\( \vec {c_{j} {}}  \) in that set, let 
\( T_{j}  \subseteq T \) be a set such that \( {s'} \in T_{j} \) iff  \( \bigwedge f_i(s') = c_{ji} \). 
That is, \( T_j \) identifies those  situations from \( T \) where fluents satisfy the vector of values \( \vec {c_{j} {}}. \) Observe  that \[
\eqref{eq:N2T1} = \sum_{j=1}^k \sum_{s' \in T_j} p(s',s). 
\]

\noindent But, of course, \( T = \bigcup T_j. \) Therefore, \[
\sum_{j=1}^k \sum_{s' \in T_j } p(s',s) = \sum_{s' \in T} p(s',s).
\]
Therefore, \eqref{eq:N1T1}  and \eqref{eq:N2T1} define the same number. \qed 
	
\end{proof}

Be that as it may, Definition \ref{defn:bel with two sums} still involves summations over situations. To arrive at a definition that eschews the summing of situations, 
we start with the case of initial situations. In this matter, we will be insisting on a precise space of initial situations. For this, let us recall the axiomatization 
of the situation calculus presented in \cite{DBLP:journals/etai/LevesquePR98}
for multiple initial situations, which includes a sentence  saying there is
precisely one initial situation for \emph{any} possible vector of fluent values.  This can be written as follows:
\begin{equation}\label{eq:stara}
[\forall{\vec{x}}\,\exists{\ivar}\!\bigwedge f_i(\ivar) = x_i]
\,\land\,
[\all \ivar, \ivar'\!.\!\bigwedge f_i(\ivar) = f_i(\ivar') 
          \supset \ivar = \ivar']
\tag{\stara}
\end{equation} 

\noindent (Recall that $i$ ranges over the indices of the fluents in $\L$, that is, $\{1, \ldots, n\}.$) Under the assumption~\eqref{eq:stara}, we can rewrite Definition \ref{defn:belief} for
$s=\ins$ as
\begin{equation}\label{eq:belief-ini}
\bel(\phi,\ins) \defeq \frac{1}{\gamma}\!
  \sum_{\vec{x}}
    \probabbr{\ivar}{\bigwedge\! f_i(\ivar)=x_i \land \phi[\ivar]}{p(\ivar,\ins)}
\tag{B0}
\end{equation}

The two abbreviations, in fact, are equivalent: 

\begin{theorem} Let \( \D \) be any basic action theory, \( \phi \) any \( \L \)-formula, and
suppose \( \D_0 \)  includes \eqref{eq:stara}. Then the abbreviations for \( \bel(\phi,\ins) \)
in Definition \ref{defn:belief} and \eqref{eq:belief-ini} define the same 
number.
	
\end{theorem}

\begin{proof} As in Theorem \ref{thm:belief two sums}, we focus on the numerators for the two abbreviations. That is, \begin{equation}\label{eq:N1T2}
	\sum_{\{s':\phi[s'\!]\}} p(s'\!,\ins)
	\tag{$\dag$}
\end{equation}
on the one hand, and \begin{equation}\label{eq:N2T2}
	\sum_{\vec{x}}
	    \probabbr{\ivar}{\bigwedge\! f_i(\ivar)=x_i \land \phi[\ivar]}{p(\ivar,\ins)}
	\tag{$\ddag$}
\end{equation}
on the other. We show that these expressions define the same number. The denominators represent a special case, and so the claim will follow.

Let \( S  \) be 
the set of initial situations. Suppose \( f\sub i \)  ranges over \( \set{c\sub i, c'\sub i, \ldots, c''\sub i} \). 
By way of  \eqref{eq:stara}, for any vector of values \( \langle c\sub 1, \ldots, c\sub n \rangle \) for the vector of fluents \( \langle f\sub 1, \ldots, f\sub n \rangle \), there exists a (unique) situation \( s \in S \) such that \(  \bigwedge f_i(s) = c_i. \) Let \( T \subseteq S \) be such that  $s\in T$ iff \(  \phi[s] \). It is  easy to see that \[
 	   \eqref{eq:N2T2} =  \sum_{s' \in T} p(s', \ins). 
\]
Since \eqref{eq:p initial constraint} ensures that \( p({s'}, \ins) > 0 \) only if \( {s'}\) is an initial situation, we get that \[
	 \eqref{eq:N1T2} = \sum_{s' \in T} p(s', \ins). 
\]
Therefore \eqref{eq:N1T2} and \eqref{eq:N2T2} define the same number. \qed 
	
\end{proof}

\noindent This shows that for $\ins$, summing over possible worlds can be
replaced by summing over fluent values. 

Unfortunately, \eqref{eq:belief-ini} is only geared for initial situations. For non-initial situations, the assumption that no two agree
on all fluent values is untenable.  To see why, imagine an action
$\xmove(z)$ that moves the robot $z$ units to the left (towards the wall) but
that the motion stops if the robot hits the wall. The successor state axiom for fluent \( \xpos \), then, might be like this:  
\begin{equation}\label{eq:xpos ssa} 
	\begin{array}{l}
		\xpos(do(a,s))  = u \equiv \\
	\qquad \mbox{} \lnot\exists{z}(a=\xmove(z)) \land u = \xpos(s) 
       \,\,\,\lor\\ 
	\qquad \mbox{} \exists{z}(a = \xmove(z) \land u = \max(0,\xpos(s) - z)).
	\end{array} %
\end{equation}

\noindent
In this case, if we have two initial situations that are identical except that
$\xpos\!=\!3$ in one and $\xpos\!=\!4$ in the other, then the two distinct successor situations
that result from doing $\xmove(4)$ would agree on all fluents (since both
would have $\xpos=0$). Ergo, we cannot sum over fluent values for non-initial situations
unless we are prepared to count some fluent values more than once. 

It turns out there is a simple way to circumvent this issue by appealing to Reiter's solution to the frame problem. Indeed, Reiter's solution 
gives us  a way of computing what holds in
non-initial situations in terms of what holds in initial ones, which 
can be used 
for computing belief at arbitrary successors of \( \ins \). More precisely, 

\begin{definition}\label{defn:discrete belief} \textbf{(Degrees of belief (reformulated).)} Let \( \phi \) be any \( \L \)-formula.  Given any sequence of ground action terms \( \alpha = [a_1, \ldots, a_k], \)  let\[
	\bel(\phi, s) \doteq \displaystyle \frac{1}{\gamma}\sum_{\vec{x}} \func(\vec x, \phi, s)
\] 
where if $s = do(\alpha,\ins)$ then  \[ \func(\vec x, \phi, do(\alpha,\ins)) \doteq  \lan 
         \ivar.\,
	      {\bigwedge\! f_i(\ivar)=x_i \,\land\, \phi[do(\alpha,\ivar)]}\,\to 
	           {p(do(\alpha,\ivar),do(\alpha,\ins))}\,\ran. \] %
\end{definition}

(As before, $i$ ranges over the indices of the fluents in $\L$.) 
To say more about how (and why) this definition works, we first note that by
\eqref{eq:p initial constraint} and \eqref{eq:p ssa}, $p$ will be $0$ unless
its two arguments share the same history.  So the $s'$ argument of $p$ in
Definition \ref{defn:belief} is \emph{expanded} and written as $do(\alpha,\ivar)$ in Definition \ref{defn:discrete belief}.  By ranging over all fluent values, we range over all
initial $\ivar$ as before, but without ever having to deal with fluent values in
non-initial situations.  Of course, we test that the $\phi$ holds and use the
$p$ weight in the appropriate non-initial situation. In particular, owing to \( p \)'s successor state axiom \eqref{eq:p ssa}, the weight for non-initial situations accounts for the likelihood of actions executed in the history. 
  We now  establish the following result: 

\begin{theorem} Let \( \D \) be any basic action theory with \eqref{eq:stara}
  initially, \( \phi \) any \( \L \)-formula, and \( \alpha \) any sequence of
  ground actions terms. Then the abbreviations for \( \bel(\phi,do(\alpha,\ins))
  \) in Definition \ref{defn:belief} and Definition  \ref{defn:discrete belief} define the
  same  number.
	
\end{theorem}

\begin{proof} As in Theorem \ref{thm:belief two sums}, we will focus on the numerators of the two abbreviations. That is, \begin{equation}\label{eq:dag}
	\sum_{\{s':\phi[s'\!]\}} p(s'\!,do(\alpha,\ins))
	\tag{$\dag$}
\end{equation} 
on the one hand, and \begin{equation}\label{eq:ddag}
	\sum_{\vec{x}} \lan 
	         \ivar.\,
		      {\bigwedge\! f_i(\ivar)=x_i \,\land\, \phi[do(\alpha,\ivar)]}\,\to 
		           {p(do(\alpha,\ivar),do(\alpha,\ins))}\,\ran.
		\tag{$\ddag$}
\end{equation}
on the other. We show that these expressions define the same number. The denominators represent a special case, and so the claim will follow. 

 Let \( S  \) be 
the set of initial situations, as determined by \eqref{eq:stara}. Let \( S^* = \{ {do}(\alpha, s')\colon s' \in S\}\).
Let \( T \subseteq S^* \)  such that $s''\in T$ iff \(  \phi[s''] \). It is   easy to see that \[
	\eqref{eq:ddag} = \sum_{ s' \in T} p(s', do(\alpha, \ins)). 
\]

\noindent On the other hand, by \eqref{eq:p initial constraint} and \eqref{eq:p ssa}, \(p(s, do(\alpha,\ins)) = 0 \) for all \( s \not\in S^*. \) This means that \[
	\eqref{eq:dag} = \sum_{s' \in T} p(s', do(\alpha, \ins)). 
\]
Therefore \eqref{eq:dag}	and \eqref{eq:ddag} define the same number. \qed 
\end{proof} %

\noindent Thus, by incorporating a simple constraint on initial situations, we
now have a notion of belief that does not require summing over situations.

Readers may notice that our reformulation only applies when we are given an explicit sequence $\alpha$
of actions, including the sensing ones. But this is just what we would expect
to be given for the \emph{projection problem} \cite{reiter2001knowledge}, where we are interested in inferring whether a formula holds after an action sequence.  
 In fact,
we can use regression on the $\phi$ and the $p$ to reduce the belief formula from Definition 
\ref{defn:discrete belief} to a formula involving initial situations
only. See \cite{BelleLevreg} for  work in this direction.

\section{From Weights to Densities} %
\label{sec:from_weights_to_densities_part_1}

The  framework presented so far is \emph{fully}  \emph{discrete}, which is to say that fluents, sensors and effectors are characterized by finite values and finite outcomes. Belief in \( \phi \), in particular, is the summing over a finite set of situations where \( \phi \) holds. We now generalize this framework. We structure our work by first  focusing on \emph{fully continuous domains}, which is to say that fluents, sensors and effectors are characterized by values and outcomes ranging over \( \real \). This section, in particular, explores the very first installment: effectors are assumed to be deterministic, but \emph{sensors} have continuous noisy error profiles. The next section, then, allows both effectors and sensors to have continuous noisy profiles. Further generalizations are deferred to Section \ref{sec:generalizations}. 

Let us begin by observing that the uncountable nature of continuous domains precludes summing over possible
situations.  
In this section, we present a new formalization of belief in terms of {\it
  integrating\/} over {fluent values}. This, in particular, is made possible by the developments in the preceding section.

Allowing real-valued fluents implies that there will be uncountably many
initial situations. 
Imagine, for example, the scenario from Figure \ref{fig:robot}, and that the fluent $h$ can now be any nonnegative real number.
Then for any nonnegative real $x$ there will be an initial situation where
$(\xpos=x)$ is true.  Suppose further that $\D_0$ includes:
\begin{equation}\label{eq:uniform 2 12 continuous}
	\p(\ivar,\ins) = \begin{cases}
		.1 & \text{if}~2 \leq \xpos(\ivar) \leq 12 \\ 
	0 & \text{otherwise}
	\end{cases} %
	\end{equation} 
\noindent which says that the true value of \( \xpos \) initially 
is drawn from a uniform distribution on the interval [2,12]. 
Then there are uncountably many
situations where $p$ is non-zero initially.  So the $p$ fluent now needs to be
understood as a \emph{density}, not as a weight. (That is, we
  now interpret \( p(s'\!,s) \) 
  as the \emph{density} of \( s' \) when the agent
  is in \( s \).) 
 In particular, 
for any $x$,
we would expect 
the initial degree of belief in the formula $(\xpos=x)$ to be $0$, but in  \( (\xpos\leq 12) \) to be 1. 

When actions enter the picture, even if deterministic, there is more to be said. Numerous subtleties arise with \( p \)  
in non-initial situations.
For example, if the robot were to do a $\xmove(4)$ there would be an 
uncountable number of situations agreeing on $\xpos=0$: namely, those where
$2\le\xpos\le4$ was true initially. In a sense, the point \( \xpos = 0 \) now
has \emph{weight}, and the degree of belief in $\xpos=0$ should be .2. On the other hand,  the
other points \( \xpos \in (0,8] \) should retain their densities. That is, belief in \( \xpos \leq 2 \) should be .4 but belief in \( \xpos = 2 \) should still be 0.
In effect,
  we have moved from a continuous to a \emph{mixed} distribution on $\xpos$. Of course, a subsequent rightward motion will retain this mixed density. For example, if the robot were to now move away by 4 units, the belief in \( \xpos = 4 \) would then be .2. 

To address the concern of belief change in continuous domains, we now present a generalization to BHL. One of the advantages of our approach is that we will not need to
specify how to handle changing densities and distributions like the ones above.  
These will 
emerge as side-effects, that is, shifting density changes will be \emph{entailed} by the action theory.

For our formulation of belief,  we first observe that we have fluents \( f_1, \dots,
f_n \) in \( \L \) as before, that take no argument other than the situation term but which now take their values from \( \real. \) Then: 
\begin{ldefn}\label{defn:continuous belief} \textbf{(Degrees of belief (continuous noisy sensors).)} Let \( \phi \) be any situation-suppressed $\L$-formula, and $\alpha = [a_ 1, \ldots, a_k]$ any ground sequence of action terms. The \emph{degree of belief} in \( \phi \) at \( s \) is  an abbreviation: \[
\bel(\phi, s) \doteq 	\displaystyle\frac{1}{\gamma}\fullintegralx
	\func(\vec x, \phi, s)
\]
where, as in Definition \ref{defn:discrete belief}, if $s = do(\alpha,\ins)$ then  \[ P(\vec x, \phi, do(\alpha,\ins)) \doteq \lan 
         \ivar.\,
	      {\bigwedge\! f_i(\ivar)=x_i \,\land\, \phi[do(\alpha,\ivar)]}\,\to 
	           {p(do(\alpha,\ivar),do(\alpha,\ins))}\,\ran.\]   
\end{ldefn}

\noindent That is, the belief in \( \phi \) is obtained by ranging over all possible fluent values, and integrating  the densities of
situations where \( \phi \) holds. If we were to compare the above definition to Definition \ref{defn:discrete belief}, we see that we have simply shifted from summing over finite domains to integrating over reals. In fact, we could read \( \func \) as the \emph{(unnormalized) density} associated with a situation satisfying \( \phi. \) As discussed, by insisting on an explicit
world history, the $\ivar$  
need only range over 
initial situations, giving us the
exact correspondence with fluent values. 
This completes our new definition of belief.  To summarize, our extension to the 
BHL scheme is defined using a few convenient abbreviations, such as for \( \bel \) and mathematical integration, and where an action theory consists of: \begin{enumerate}
	\item \( \D_0 \) (with \eqref{eq:p initial constraint}) as usual, but now also including \eqref{eq:stara}; 

	\item precondition axioms as usual; 

	\item successor state axioms, including one for \( p \), namely \eqref{eq:p ssa},  as usual; 

	\item foundational domain-independent axioms as usual; and 

	\item  action likelihood axioms, one for each action type.

\end{enumerate}

Note that, apart from \eqref{eq:stara} and \( \bel \)'s new abbreviation, we carry over precisely the same components as would BHL. By and large, the extension, thus, retains the simplicity  of their proposal, and comes with minor additions.  We will show that it has reasonable properties using an example and its
connection to Bayesian conditioning below. 

In the sequel, we assume, without explicitly mentioning so, that basic action theories include the sentences \eqref{eq:p initial constraint}, \eqref{eq:p ssa} and \eqref{eq:stara}.

\subsection{Bayesian Conditioning} %
\label{sub:belief_change_and_bayesian_conditioning}

We now explicate the relationship between our definition for $\bel$ and  \emph{Bayesian
  conditioning} \cite{pearl1988probabilistic}. Bayesian conditioning is a standard model for belief change wrt noisy sensing \cite{thrun2005probabilistic} and it rests on two significant
assumptions. First, sensors do not physically change the world, and second,
conditioning on a random variable \( f \) is the same as conditioning on the
event of observing  \( f \). 

In general, in the language of the situation calculus, there need not be a distinction between sensing actions and physical actions. In that case, the agent's beliefs are affected by the sensed value as well as any other physical changes that the action might enable to adequately capture the ``total evidence'' requirement of Bayesian conditioning. 

The second assumption expects that sensors only depend on the true value for the fluent. For example, in the formulation of \eqref{eq:sonar error profile simple}  the sonar's error profile is determined solely by \( \xpos \). But to suggest that the error profile might depend on other factors about the environment, as formulated by \eqref{eq:sonar error profile complicated} for example, goes beyond this simplified view. In fact, here, the agent also learns about the room temperature, apart from sensing the value of \( \xpos. \) 

Thus, our theory of action admits a view of dynamical systems far richer than the standard setting where Bayesian conditioning is applied. Be that as it may, when a similar set of assumptions are imposed as
axioms in an action theory, we obtain a sensor fusion model identical to Bayesian conditioning. This connection was demonstrated in BHL for the discrete case. We prove the property formally for  continuous variables below. 

We begin by stipulating that actions are either of the \emph{physical} type or of the
\emph{sensing} type \cite{citeulike:528170},  
the latter being the kind that
do not change the value of any fluent, that is, such actions do not appear in the successor state axioms for any fluent.  
Now, if \( \ti{obs}(z) \) senses the
true value of fluent \( f \), then assume the sensor error model to be: \[ l(\ti{obs}(z),s)=u
\,\,\equiv\,\, u=\ti{Err}(z,f(s))
\]
where \( \ti{Err}(u_1,u_2) \) is some expression with only 
two free variables,
both numeric. This captures the notion established above:  the error model of a sensor
measuring \( f \) depends only on the true value of \( f \), and is
independent of other factors. Finally, for simplicity, assume \emph{obs(z)} is always executable: \[
	\poss(\mathit{obs(z)},s) \equiv \true. 
\]
Then we obtain:  

\begin{theorem}\label{thm:bayes theorem} Suppose \( \D \) is any basic action theory with likelihood and precondition axioms for \( \mathit{obs(z)} \) as above, \( \phi \) is any \( \L \)-formula mentioning only \( f \),  and  \( u\in \{ x_1, \ldots, x_n \} \) that \( f \) takes a value from. Then we obtain\emph{:}  
    \[
	\begin{array}{l}
	\D\models \bel(\phi, do(\ti{obs}(z),\ins)) = \displaystyle\frac{ \int_{\vec x } 
		[\func(\vec x, \phi \land f = u, \ins) \times \ti{Err}(z,u)]}{
		\int_{\vec x } 
		[\func(\vec x, f = u, \ins) \times \ti{Err}(z,u)] }
	\end{array}
\] 
	
\end{theorem}

\noindent    That is, the \emph{posterior belief} in \( \phi \) is obtained from the \emph{{prior} density} and the error likelihoods for all points where \( \phi \) holds given that \( z \) is observed, normalized over all points. 
The proof for the theorem is as follows. \\

\begin{proof} Without loss of generality, assume \( f \) takes the value \( x \), and the remaining fluents are \( f_1, \ldots f_n \) that take values from \( x_1, \ldots, x_n. \) Let \( a \) denote \( \mathit{obs(z)}. \)
	 From Definition \ref{defn:continuous belief}, \( \bel(\phi, do(\ti{obs(z)},\ins))  \) is an abbreviation for: \[
	 \begin{array}{l}
	\dis	\frac{1}{\gamma} \displaystyle \int_{\vec x} \underline{\func(\vec x, \phi, do(a,\ins))} =  \hfill \text{(a)} \\[3ex]
\dis \frac{1}{\gamma} \fullx \probabbr{\ivar}{f(\ivar) = x \land \bigwedge_{i=1}^n f_i(\ivar) = x_i \land \underline{\phi[do(a,\ivar)]}}{p(do(a,\ivar), do(a, \ins))} = \hfill \text{(b)} \\[3ex] 
\dis \frac{1}{\gamma} \fullx \probabbr{\ivar}{f(\ivar) = x \land \bigwedge f_i(\ivar) = x_i \land \underline{\phi[\ivar]}}{{p(do(a,\ivar), do(a, \ins))}} = \hfill \text{(c)} \\[3ex] 
\dis \frac{1}{\gamma} \fullx \probabbr{\ivar}{f(\ivar) = x \land \bigwedge f_i(\ivar) = x_i \land {(\phi \land f = x)[\ivar]}}{\underline{p(do(a,\ivar), do(a, \ins))}} = \hfill \hspace{7em} \text{(d)} \\[3ex] 
\dis \frac{1}{\gamma} \fullx \probabbr{\ivar}{f(\ivar) = x \land \bigwedge f_i(\ivar) = x_i \land (\phi\land f=x)[\ivar]}{\underline{p(\ivar,\ins)\times \ti{Err}(z,x)}} = \hfill \text{(e)} \\[3ex] 
\dis \frac{1}{\gamma} \fullx Err(z,x) \times  \underline{\probabbr{\ivar}{f(\ivar) = x \land \bigwedge f_i(\ivar) = x_i \land (\phi\land f=x)[\ivar]}{{p(\ivar,\ins)}}} = \hfill \text{(f)} \\[3ex] 
\dis \frac{1}{\gamma} \fullx Err(z,x) \times  \func(\vec x, \phi \land f = x, \ins).
	 \end{array} 
\]
The arguments underline those parts of the sentences that are being reduced. Step (a) expands \( P. \) In step (b), owing to the fact that sensing actions do not change fluent values (by our first assumption tailored to  Bayesian conditioning), \( \phi[do(a,\ivar)] \) is equivalent to \( \phi[\ivar]. \) In step (c), we observe that what appears to the left of the right-arrow in step (b) is equivalent to one where \( \phi[\ivar] \) is replaced by \( (\phi\land f=x)[\ivar]. \) (The main reason for introducing the formula \( f=x \) is to allow us to identify those situations where \( f \) takes the value \( x \) which are all to be multiplied by the error likelihood  for observing \( z \) when the true value is \( x. \)) In step (d), we use \eqref{eq:p ssa} and the fact that \( a \) is always executable to replace \( p(do(a,\ivar),do(a,\ins)) \) by \( p(\ivar,\ins) \times l(a,\ivar) = p(\ivar,\ins) \times \mathit{Err}(z,x). \) Now note that the term appearing in the context of an integral is suggesting that if there is an initial situation where a particular condition holds, then a certain value is returned, and otherwise 0 is returned. This allows us to place the term \( Err(z,x) \) on the outside in step (f), giving us the required numerator that appears in the claim. The expansion of the denominator is analogous with \( \true \) being a special case for \( \phi \), and so we are done. \qed

\end{proof}

The usual case for posteriors are formulas such as \( b \leq f \leq c \), 
which is estimated from the prior and error likelihoods for all points in the range  \( [b,c] \), as demonstrated by the following consequence: 

\begin{corollary} Suppose \( D \) is any basic action theory with likelihood and precondition axioms for \( \mathit{obs(z)} \) as above, \( f \) is any \( \L \)-fluent, and \( u \) is a variable from \( \vec x = \la x_1, \ldots, x_m \ra \) that \( f \) takes a value from. Then we obtain\emph{:}     \[
	\begin{array}{l}
	\D\models \bel( a\leq f \leq b, do(\ti{obs}(z),\ins)) = \displaystyle\frac{ \int_{\vec x } 
		[\func(\vec x, f = u \land a \leq u \leq b, \ins) \times \ti{Err}(z,u)]}{
		\int_{\vec x } 
		[\func(\vec x, f = u, \ins) \times \ti{Err}(z,u)] }
	\end{array}
\] 
\end{corollary} 

More generally however, and unlike many probabilistic formalisms, we are able
to reason about any logical property $\phi$ of the random variable $f$ being
measured.

\subsection{Example} %
\label{sub:example_first}

Using an example, we demonstrate the formalism and Theorem \ref{thm:bayes theorem} in particular. To reason about the beliefs of our robot, let us build a simple basic action
theory $\D$. We extend the setting from Figure \ref{fig:robot} to a 2-dimensional grid, as shown in Figure \ref{fig:robot-v}. As before, let \( \xpos \) be the fluent denoting its horizontal position (that is, its distance to the wall), and let the robot's 
vertical position be given by a fluent \( \ypos \). The components of \( \D \) are as below. 

\begin{figure}[h]
  \centering
    \includegraphics[width=.4\textwidth]{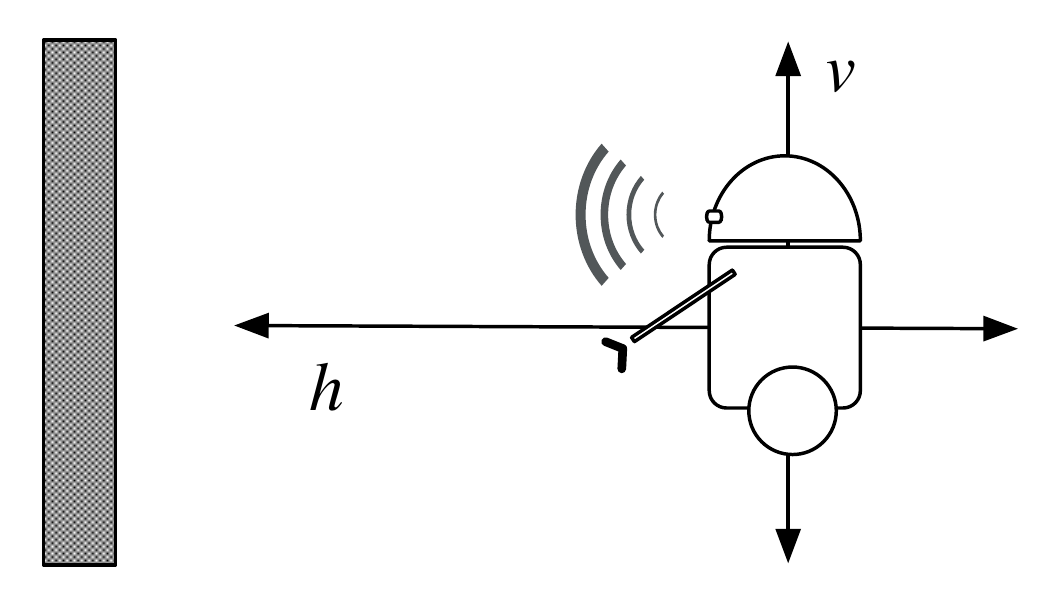}
  \caption{A robot in a 2-dimensional grid.}
  \label{fig:robot-v}
\end{figure}

\begin{itemize}
	\item Imagine a \( p \) of the form: 
	\begin{equation}\label{eq:initial p spec}
		p(\ivar,\ins) = \begin{cases}
			.1 \times \N(\ypos(\ivar);0,16) & \text{if } 2\leq \xpos(\ivar) \leq 12 \\
			0 & \text{otherwise}
		\end{cases} %
	\end{equation} 

	\noindent This says that the value of \( \ypos \) is normally distributed
	about the horizontal axis with variance 16, and independently, that the 
	value of 
	\( \xpos \)
	is uniformly distributed between 2 and 12. 

	Note also that initial beliefs can  
	  be specified for $\D_0$ using $\bel$ directly. For example, to express that the true value of \( \xpos \) is believed to be uniformly distributed on the interval $[2,12]$ we might equivalently include the following theory in \( \D_0 \): \[
	  	\{ \bel(2 \leq \xpos \leq 12, \ins) = .1, \bel(\xpos \leq 2 ~\lor ~ \xpos \geq 12, \ins) = 0 \},   
	  \]
	  and analogously for the fluent $\ypos$.

	For this example, a simple distribution has been chosen for illustrative purposes. In general, recall from Section \ref{sub:discussion} that the \( p \)-specification does not require  the variables  to be independent, nor does it have to mention all  variables.

\item 	For simplicity, let us assume that actions are always executable, \ie that \( \D \) includes \begin{equation}\label{eq:poss always true}
	\poss(a,s) \equiv \true
\end{equation}  for all actions \( a \). For this example, we assume three action types:  action $\xmove(z)$ that moves the robot $z$ units towards the wall, action $\ymove(z)$ that moves the robot $z$ units away from the horizontal axis, and action $\xsense(z)$ that gives a reading of $z$ for the distance between the robot and the wall.

\item 	The successor state axiom for \( \xpos \) is as in  \eqref{eq:xpos ssa}, and the one for $\ypos$ is as follows:

\begin{equation}\label{eq:ypos ssa}
\begin{array}{l}
 \ypos(do(a,s)) = u \equiv  \lnot\exists{z} (a=\ymove(z)) \land 
u =\ypos(s) ~\lor \\
\qquad \mbox{}\qquad \mbox{}\qquad \mbox{}\qquad \mbox{} \exists{z} (a = \ymove(z) \land u = \ypos(s) + z).
\end{array}%
\end{equation}

	\item 	For the sensor device, suppose its error model is given as follows: \begin{equation}\label{eq:sonar error model}
		l(\xsense(z),s) = u \equiv (z\geq 0\land u = \N(z-\xpos(s);0,4)) \lor (z<0 \land u=0).
	\end{equation}

	The error model says that for nonnegative \( z \) readings, the difference between the reading and the true value is normally distributed with mean 0 (which indicates that there is no systematic bias) and variance 4.\footnote{For a more elaborate example involving multiple competing sensors and systematic bias, see \cite{bellelevcs}.} 
	
	For the remaining (physical) actions, we let \begin{equation}\label{eq:likelihood deterministic physical actions}
		l(\xmove(z),s) = 1,~~ l(\ymove(z),s)  =1
	\end{equation} since they are assumed to be deterministic for this section.
	
\end{itemize}

Then we obtain: 
\begin{theorem}\label{thm:example} 
Let \( \D \) be a basic action theory that is the union of \( \{ \eqref{eq:initial p spec}, \eqref{eq:poss always true}, \eqref{eq:xpos ssa}, \eqref{eq:ypos ssa}, \eqref{eq:sonar error model}, \eqref{eq:likelihood deterministic physical actions} \} \). Then 
the following are logical entailments of \( \D \):

\end{theorem}

\begin{enumerate}

	\item\label{i:point 4 has 0 weight} \( \bel([\xpos = 3 \lor \xpos = 4 \lor
      \xpos = 7],\ins) = 0 \).
	
To see how this follows, let us begin by expanding \( 	\bel([\xpos = 3 \lor \xpos = 4 \lor
      \xpos = 7],\ins) \): \[
      	\dis \frac{1}{\gamma} \fullx \func(\vec x, \xpos = 3 \lor \xpos = 4 \lor \xpos = 5, \ins). 
\tag{a}
      \]
For the rest of the section, let \( \xpos \) take its value from \( \xo \) and \( \ypos \) take its value from \( \xt. \) By means of \eqref{eq:stara},  there is  exactly one situation for any set of values for \( \xO \) and \( \xT. \) The \( \func \) term for any such situation, however, is 0 unless \( \xpos = 3 \lor \xpos = 4 \lor \xpos =5 \) holds at the situation. 
Thus, (a) basically simplifies to: \[
	\dis \frac{1}{\gamma} \fullx \begin{cases}
		.1 \times \N(\xt;0,16) & \text{if \( \xo\in \{3,4,5\} \)} \\ 
		0 & \text{otherwise}
	\end{cases} \hspace{3em} = 0.
\]
	
In effect, although we are integrating a  function $\delta(x_1,x_2)$ over all real
values, $\delta(x_1,x_2)=0$ unless 
$x_1\in\{3,4,7\}.$

	\item\label{i:initial belief about h} \( \bel(\xpos \leq 9,\ins) = .7 \).
	
	We might contrast this with the previous property in that for any given value for \( \xo \) and \( \xt \), the \( \func \) term is 0 when \( \xo >9. \) When \( \xo\leq 9, \) however, the \( p \) value for the situation is obtained from the specification given by \eqref{eq:initial p spec}. That is, we have: \[\begin{array}{l}
	\dis	\frac{1}{\gamma}	\fullx \begin{cases} \text{Initial specification given by \eqref{eq:initial p spec}} & \text{if \( \xpos\leq 9 \)} \\
	0 & \text{otherwise} 
		\end{cases}	 \hspace{9em} \hfill =  \\[3ex]
			\dis \frac{1}{\gamma} \fullx \begin{cases} .1 \times \N(\xt; 0, 16) & \text{if  \(\thereis \ivar.~\xpos(\ivar) = \xo, \xo \in [2,12]  \) and \(  \xpos(\ivar) \leq 9 \)} \\
			0 & \text{otherwise}
			\end{cases}	
	 \hfill =\\[3ex]
	\dis \frac{1}{\gamma} \int_\real \int_2^9 .1 \times \N(x_2;0,16) ~\dummy x_1~
	\dummy x_2 
	\end{array}  
	\] 
The numerator evaluates to .7, and the denominator to 1. 

\item \( \bel(\xpos > 7\ypos,\ins) \approx .6. \) 

Beliefs about 
any  mathematical expression involving   the random variables, even when that does not correspond to well known  density functions, are  entailed. To evaluate this one, for example, observe that we have \[\begin{array}{l}
	\dis \frac{1}{\gamma} \fullx \begin{cases} .1\times \N(\xt;0,16) & \text{if \( \xo \in [2,12] \) and \( \xo > 7 \xt \)} \\
	0 & \text{otherwise}
	\end{cases} \hspace{3em} =  \\[3ex] 
	\dis \frac{1}{\gamma} \int_2^{12} \int_{-\infinity}^{\xo /7} .1\times \N(\xt;0,16) \dummy \xo.
\end{array}	
\]

	\item\label{i:weight for 0} \( \bel(\xpos = 0, do(\xmove(4),\ins)) = .2 \).
	
	 Here a \emph{continuous} distribution evolves into a \emph{mixed}
     distribution. This results from  \(  \bel(\xpos = 0, do(\xmove(4),\ins)) \) first expanding  as: \[
     	\dis \frac{1}{\gamma} \fullx \func(\vec x, \xpos =0, do(\xmove(4),\ins)) \tag{a}
     \]
The \( P \) term, then, simplifies to: \[
	\probabbr{\ivar}{\bigwedge f(\ivar) = x \land (\xpos = 0)[do(\xmove(4),\ivar)]}{p(\ivar,\ins)} \tag{b}
\]
That is, since \( \xmove(z) \) has no error component, \( l(\xmove(z),s) = 1 \) for any \( s \) in accordance with \( \D. \) Therefore, \( p(do(a,\ivar), do(a, \ins))  = p(\ivar,\ins). \) Now (b) says that for every possible value for \( \xpos \) and \( \ypos \), if there is an initial situation where \( \xpos = 0 \) holds after moving leftwards, then its \( p \) value is to be considered. Note that for any initial situation \( s \) where \( \xpos(s) \in [2,4] \), we get \( \xpos(do(\xmove(4),s)) = 0 \) by \eqref{eq:xpos ssa}. This leaves us with: \[
	\begin{array}{l} 
		\dis \frac{1}{\gamma} \fullx \begin{cases}
			.1 \times \N(\xt;0,16) & \text{if \( \thereis \ivar.~\xpos(\ivar) =\xo, \xo\in [2,12], \xpos(\ivar) \in [2,4] \)} \\ 
			0 & \text{otherwise}
		\end{cases}  \qquad \mbox{} \qquad \mbox{} = \\[3ex]
		\dis \frac{1}{\gamma} \fullx \begin{cases}
			.1 \times \N(\xt;0,16) & \text{if \( \xo \in [2,4] \)} \\ 
			0 & \text{otherwise}
		\end{cases} 
	\end{array} \tag{c}
\]
We can show that \( \gamma =1 \), which means (c) = .2. This change in beliefs is shown in Figure \ref{fig:xposmixed}.

	\item \( \bel(\xpos \leq 3, do(\xmove(4),\ins)) = .5 \).
	
	\( \bel \)'s definition is  amenable to a set of $h$ values, where one
    value has a weight of $.2$, and all the other real values have a uniformly
    distributed density of $.1$.

\item \(\bel([\exists{a,s}.~now\!=\!do(a,s) \land \xpos(s)\!>\!1],
do(\xmove(4),\ins)) = 1\).

It is possible to refer to earlier or later situations 
using $now$ as the current situation.  This says that after moving,
there is full belief that $(\xpos>1)$ held before the action.

	\item\label{i:moving towards moving away}  \( \bel(h=4,do([\xmove(4),\xmove(-4)],\ins)) = .2 \) 
	
	\( \bel(\xpos=4,do([\xmove(-4),\xmove(4)],\ins)) = 0 \).

	The point \( \xpos \!=\! 4  \) has 0 weight initially, as shown in 
	item \ref{i:point 4 has 0 weight}. Roughly,  if the agent were to move leftwards first then many points would ``collapse'', as shown in item \ref{i:weight for 0}. The point would then obtain a \( \xpos \) value of 0, and have a weight of .2. The weight is then retained on moving away by 4 units, where the point once again gets \( \xpos \) value 4. On the other hand, if this entire phenomena were reversed then none of these features are observed because the collapsing does not occur and the entire space remains fully continuous.

 \item \( \bel(-1 \leq \ypos \leq 1, do(\xmove(4),\ins)) = \bel(-1 \leq \ypos \leq
   1, \ins) = \smallint_{-1}^1 \N(x_2;0,16)  \dummy x_2 \).

Owing to Reiter's solution to the frame problem, belief in \( \ypos \) is
unaffected by a lateral motion. That is, a leftwards motion does not change \( \ypos \) in accordance with \eqref{eq:ypos ssa}. As per \eqref{eq:initial p spec}, the initial belief in \( \ypos \in [-1,1] \)  is the area
between \( [-1,1] \) bounded by the specified Gaussian.

\item \( \bel(\ypos \!\leq\! 7, do(\ymove(2.5),\ins)) = \bel( \ypos \!\leq\! 4.5, \ins) \).

After the action $\ymove(2.5)$, the Gaussian for \( \ypos \)'s value has
its mean ``shifted'' by 2.5 because the density associated with \( \ypos = x_2
\) initially is now associated with \( \ypos = x_2+2.5. \) Intuitively, we have: \[
	\begin{array}{l}
		\dis \frac{1}{\gamma} \fullx \begin{cases}
		.1 \times \N(\xt;0,16) & \text{if \( \thereis \ivar.\ypos(\ivar) = \xt \) and \(  (\ypos \leq 7)[do(\ymove(2.5),\ivar)] \)} \\ 
		0 & \text{otherwise}
		\end{cases} \qquad \mbox{} \qquad \mbox{}  = \\[3ex]
		\dis \frac{1}{\gamma} \fullx \begin{cases}
		.1 \times \N(\xt;0,16) & \text{if \( \thereis \ivar.\ypos(\ivar) = \xt \) and \(  \ypos(\ivar) \leq 4.5 \)} \\ 
		0 & \text{otherwise}
		\end{cases}	\hfill = \\[3ex]
		\dis \frac{1}{\gamma} \fullx \begin{cases}
		.1 \times \N(\xt;0,16) & \text{if \( \xt \leq 4.5 \)} \\ 
		0 & \text{otherwise}
		\end{cases}
	\end{array}
\]

\item\label{i:first reading} \( \bel(\xpos \leq 9, do(\xsense(5),\ins)) \approx .97 \).

\( \bel(\xpos \leq 9, do([\xsense(5),\xsense(5)],\ins) \approx .99 \).

Compared to item \ref{i:initial belief about h}, belief in \( \xpos \leq
9\) is sharpened by obtaining a reading of $5$ on the sonar, and sharpened to
almost certainty on a second reading of 5.  This is because the
$p$ function, according to \eqref{eq:p ssa}, incorporates the
likelihood of each $\xsense(5)$ action. More precisely, the  belief term in the first entailment simplifies to: \[
	\dis \frac{1}{\gamma} \fullx \probabbr{\ivar}{\xpos(\ivar) = \xo \land \ypos(\ivar) =\xt \land \xpos(\ivar) \leq 9}{p(\ivar,\ins) \times \N(5-\xo;0,4)} \tag{a}
\]
Note that we have replaced \( (\xpos \leq 9)[do(\xsense(5),\ins)] \) by \( (\xpos \leq 9)[\ivar] \) since \( \xsense(z) \) does not affect \( \xpos. \) From (a), we get \[
		\dis \frac{1}{\gamma} \fullx \N(5-\xo;0,4)\times  \probabbr{\ivar}{\xpos(\ivar) = \xo \land \ypos(\ivar) =\xt \land \xpos(\ivar) \leq 9}{p(\ivar,\ins)} \tag{b}
\]
We know from \eqref{eq:initial p spec} that those initial situations where \( \xpos \leq 2 \) have \( p \) values 0. Therefore, from (b), we get: \[\begin{array}{l}
	\dis \frac{1}{\gamma} \fullx \N(5-\xo;0,4)\times \begin{cases}
		.1 \times \N(\xt;0,16) & \text{if \( \xo \in [2,9] \)} \\ 
		0 & \textrm{otherwise}
	\end{cases} \hspace{3em} =\\[3ex]
	\dis \frac{1}{\gamma} \int_\real \int_2^9 \N(5-\xo;0,4)\times .1 \times \N(\xt;0,16)  \dummy {\xo} \dummy {\xt}
\end{array}
\]
After a second reading of 5 from the sonar, the expansion for belief is analogous, except that the function to be integrated gets multiplied by a second $\N(5-x_1;0,4)$ term. It is then not hard to see that belief sharpens significantly with this multiplicand. The agent's changing densities are shown  in 
Figure~\ref{fig:sonargraph}.

\end{enumerate}

\begin{figure}[tp]
	\centering
		\includegraphics[width=9cm]{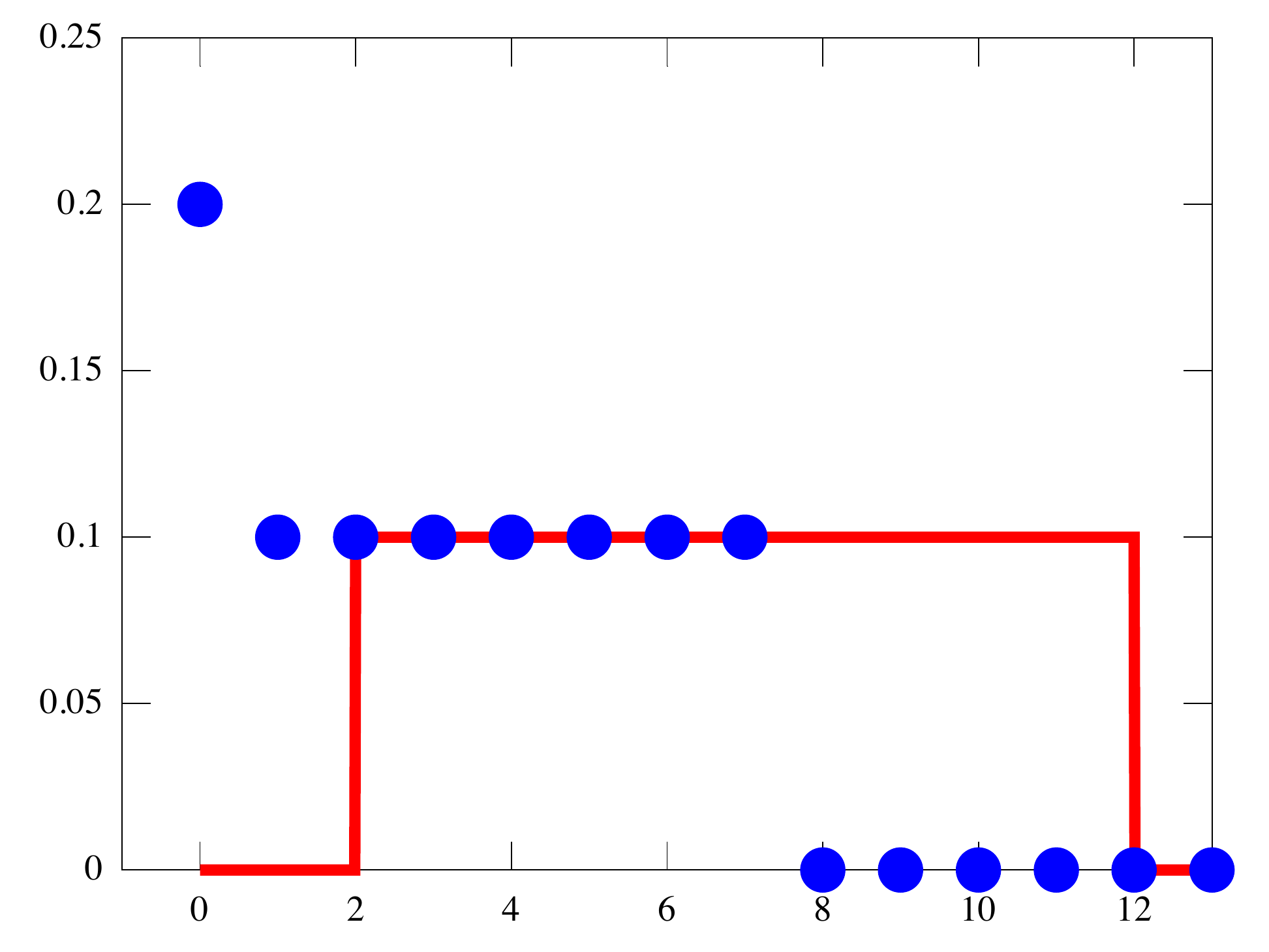}
	\caption{Belief update for \( \xpos \) after physical actions. Initial belief at \( \ins \) (in solid red) and after a leftward move of 4 units (in blue with point markers).}
	\label{fig:xposmixed}
\end{figure}

\begin{figure}[tp]
\begin{center}
\includegraphics[width=8cm]{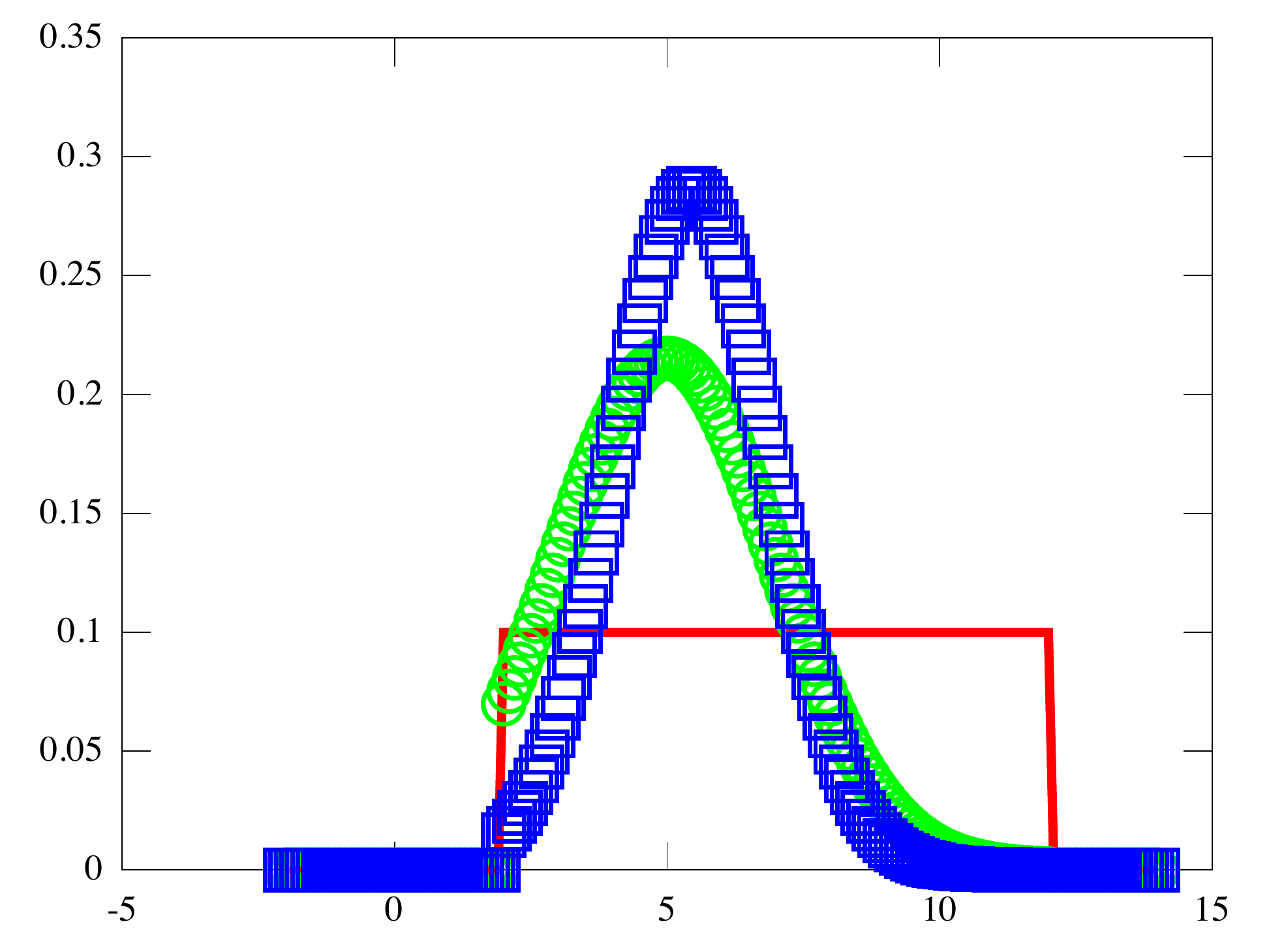}%
\caption{ Belief change for \( \xpos \) at \( \ins \) (in solid red), after
  sensing 5 (in green with circular markers), and after sensing 5 twice (in blue with square markers).} 
\label{fig:sonargraph}
\end{center}
\end{figure}

\section{Noisy Acting} %
\label{sec:from_weights_to_densities_part_2}

In the presentation so far, we assumed physical actions to be deterministic. By this we mean that when a physical action \( a \) occurs, it is clear to us (as modelers) but also the agent how the world has changed on \( a. \) Of course, in realistic domains, especially robotic applications, this is \emph{not} the usual case. In this section, in a domain that has continuous fluents, we show how our current account of belief can be extended to reason with  sensors as well as effectors that are noisy. 

In line with the rest of this work, effector noise is given a quantitative account. Let us first reflect on what is expected with noisy acting. When an agent senses, as in the case of \( \xsense(z), \) the argument for this action is not chosen by the agent. That is, the world decides what \( z \) should be, and based on this  reading of \( z \), the agent comes to certain conclusions about its own state. The noise factor, then, simply addresses the phenomena that the number \( z \) returned may differ from the true value of whatever fluent the sensor is measuring. 

Noisy acting diverges from that picture in the following sense. The agent \emph{intends} to do action \( a \), but what actually occurs is  \( a' \) that is possibly different from \( a. \) For example, the agent may want to move 3 units, but, unbeknownst to the agent, it may move  by 3.042 units. The agent, of course, does not observe this outcome. Nevertheless, provided the agent has an account of its effector's inaccuracies, 
it is reasonable for the agent to believe that it  is in fact closer to the wall, even if it may not be able to precisely tell by how much. Intuitively, the result of a nondeterministic action is that any number of successor situations might be obtained, which are all indistinguishable in the agent's perspective (until sensing is performed).  Depending on the likelihoods of the action's potential outcomes, some of these successor situations are considered more probable  than others. The agent's belief about what holds then must incorporate these relative likelihoods.  So, in our view, nondeterminism is really an \emph{epistemic} notion.

\subsection{Noisy Action Types}

Following \cite{Bacchus1999171}, perhaps the simplest extension to make all this precise is to assume that deterministic actions such as \( \xmove(x) \) now have companion action types \( \xmove(x,y) \) in \( \L \). The intuition is that \( x \) represents the \emph{nominal} value, which is the number of units that the agent intends to move, while \( y \) represents the actual distance moved. 
The actual value of \( y	 \) in any ground action, of course, is not observable for the agent. 
 This simple idea will need three adjustments to our account: \begin{enumerate}
	\item successor state axioms need to be built using these new action types;   
	\item the formalism must allow the modeler to formalize that certain outcomes are more likely than others, that is, noisy actions may be associated with a probabilistic account of the various outcomes; and 
	\item the notion of belief must incorporate the nominal value, the range of possible outcomes and their likelihoods.  
\end{enumerate}

First, we address successor state axioms. These are now specified as usual, but using the second argument, which is the actual outcome, rather than   the nominal value, which  is ignored. For example, for the fluent \( \xpos \), instead of \eqref{eq:xpos ssa}, we will now have: \begin{equation}\label{eq:xpos ssa noisy} 
\begin{array}{l}
\xpos(do(a, s)) = u \equiv \exists x, y[a = \move(x,y) \land u = \max(0, \xpos - y)]	 ~\lor \\ 
\qquad \mbox{} \qquad \mbox{} \qquad \mbox{}~~~~ \neg \exists x, y[a = \move(x,y)] \land u = \xpos(s).
\end{array}
\end{equation}

The reason for this modification is obvious.  If \( y \) is the actual outcome then the fluent change should be contingent on this value rather than what was intended. It is important to note that no adjustment to the existing (Reiter's) solution to the frame problem is necessary.  

\subsection{The \golog\ approach}

The foremost issue now is to use the above idea to allow for more than one possible successor situation. Clearly, we do not want the agent to control the actual outcome in general. So the approach taken by BHL is to think of picking the second argument as a \emph{nondeterministic}  \golog\  {program} \cite{reiter2001knowledge}. Briefly, \golog\ is an agent programming proposal where one is allowed to formulate complex actions that denote sequential and nondeterministic executions of actions, among others, and is essentially a basic action theory. 
Given the action \( \move(x,y) \), for example, the  \golog\ program \( \textsc{Move(x)} \) might stand for the abbreviation \( \pi y.~\move(x,y) \), which corresponds to a ground action \( \move(x,n) \) where \( n \) is chosen nondeterministically. For our purposes, we would then imagine that the agent  executes \golog\  programs. 

There are some advantages to this approach: namely, we only have to look at the logical entailments, including ones mentioning \( \bel, \) of such \golog\ programs. Since traces of these programs  account for many potential outcomes, \( \bel \) does the right thing and accommodates all of these when considering knowledge.  
But the disadvantage is that the resulting formal specification turns out be unnecessarily complex, at least as far as projection is concerned. 

For  projection tasks, we show that we can settle on a simpler alternative, one that does not appeal to \golog. Like BHL,  we assume that the world is deterministic, where the result of doing a ground action leads to a distinct successor. Roughly, the intuition then is that when a noisy action is performed, the various outcomes of the action as well as the potential successor situations that are obtained wrt these are treated at the level of belief. 

\subsection{Alternate Action Axioms}

Inspired by \cite{Delgrande:2012fk}, our approach is based on the introduction of a distinguished predicate $\alt.$  The idea is this: if \( \alt(a, a', \vec z) \) holds for ground action \( a  = A(\vec c) \)   then we understand this to mean that the agent believes that any  instance of \( a' = A(\vec z) \) might have been executed instead of \( a. \) Here, \( \vec z \) denotes the range of the arguments for potential outcomes.

To see how that gets used with the new action types such as \( \xmove(x,y) \), 
consider the ground action \( \move(3,3.1) \). So, the agent intends to move by 3 units but what has actually occurred is a move by 3.1 units. Since the agent does not observe the latter argument, from its  perspective,  what  occurred could have been a move by 2.9 units, but also perhaps (although less likely) a move by 9 units. Thus, the ground actions \( \move(3,2.9) \) and  \( \move(3,9) \) are \( \alt \)-related to \( \move(3,3.1). \) (The likelihoods for these may vary, of course.) In logical terms, we might have have an axiom of the following form in the background theory: \begin{equation}\label{eq:alt move action}
		\alt(\move(x,y),a', z) \equiv a' = \move(x, z). 
\end{equation} 
This to be read as saying that \( \move(3,z) \) for every \(  z \in \real\)  are alternatives to \( \move(3,3.1) \). If we required that \( z \) is only a certain range from \( 3 \), for example, we might have: \[
	\alt(\move(x,y),a', z) \equiv a' = \move(x,z) \land | x-z| \leq c
\]
where \( c \) bounds the magnitude of the maximal possible error. On the other hand, for actions such as \( \xsense(z) \), which do not have any alternatives, we simply write: \begin{equation}\label{eq:alt sonar}
	\alt(\xsense(x), a', z ) \equiv a' = \xsense(z) \land z=x.
\end{equation}
This says that \( \xsense(x) \) is only \( \alt \)-related to itself. 

With this simple technical device, one can now additionally  constrain the likelihood of various outcomes using \( l \). For example: \beq\label{eq:simple noisy move}
		l(\move(x,y), s) = u \equiv u = \N(y-x; \mu, \sigma^2)
\eeq
says that the difference between nominal value and the actual value is normally distributed with mean \( \mu \) and variance \( \sigma^2 \). This essentially corresponds to the standard additive Gaussian noise model in robotics \cite{thrun2005probabilistic}.

To see an example of how, say, \eqref{eq:alt move action} and \eqref{eq:simple noisy move} work together with the successor state axiom \eqref{eq:p ssa} for \( p \), consider three situations \( s, s' \) and \( s'' \) associated with the same density, as shown in Figure \ref{fig:situations}. Suppose their \( h \) values are \( 6, 6.1 \) and \( 5.9 \) respectively. After attempting to move 3 units, the action  \( \xmove(3,z) \) for any \( z\in\real \) may have occurred. So, for each of the three situations, we explore successors from different values for \( z. \) Assume the motion effector is defined by a mean \( \mu=0 \) and variance \( \sigma\su 2 = 1. \) Then, the \( p \)-value of the situation \( do(\xmove(3,5.7),s') \), for example, is obtained from the \( p \)-value for \( s' \) multiplied by the likelihood of \( \xmove(3,5.7) \), which is \( \N(5.7-3;0,1) = \N(2.7;0,1) \). Thus, the successor situation \( do(\xmove(3,5.7),s') \) is much less likely than the successor situation \( do(\xmove(3,3),s) \), as should be the case.

\begin{figure}[htbp]
	\centering
		\includegraphics[height=5cm]{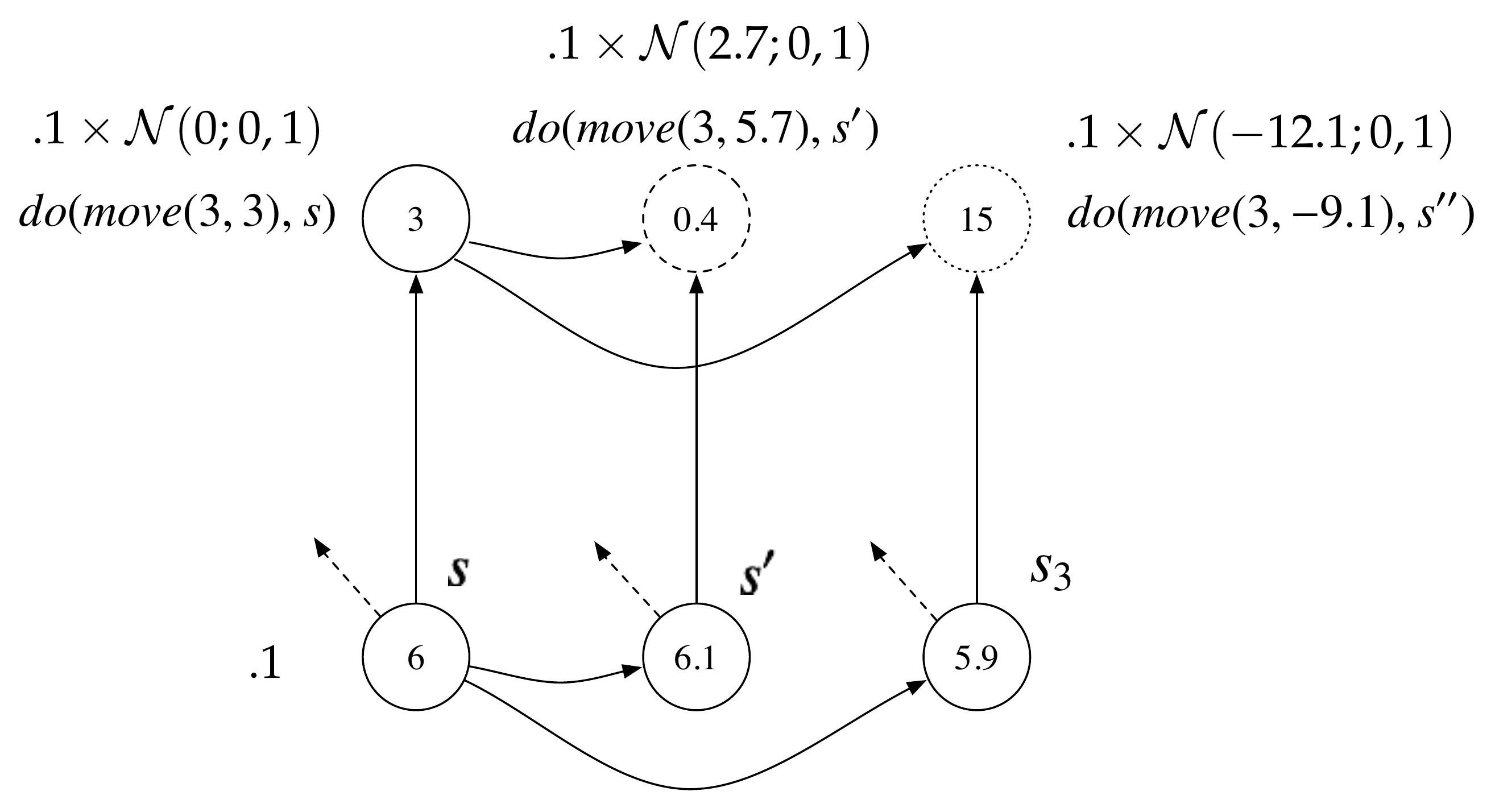}
	\caption{Situations with accessibility relations after noisy acting. The numbers inside the circles denote the \( h \) values at these situations. Dotted circles denote a lower density in relation to their epistemic alternatives.}
	\label{fig:situations}
\end{figure}

In general, we  define \emph{alternate actions} axioms that are to be a part of the  basic action theory henceforth: 

\begin{definition} Let \( A(\vec x, \vec y) \) be any action. \emph{Alternate actions} \emph{axioms} are sentences of the form: \[
	\alt(A(\vec x, \vec y), a', \vec z) \equiv a' = A(\vec x, \vec z) \land \psi(\vec x, \vec y, \vec z)
\]
where \( \psi \) is a formula that characterizes the relationship between the nominal and true values. 
\end{definition}

The one limitation with this definition is that only actions of the same \emph{type}, \ie built from the same  function symbol, are alternatives to each other. This does not allow, for example, situations where the agent intends a physical move, but instead unlocks the door. 
Nevertheless, this definition is not unreasonable because noisy actions in robotic applications typically involve additive noise  \cite{thrun2005probabilistic}. Moreover, this limitation only assists us in arriving at a simple and  familiar definition for belief. A more involved definition would allow for other variants.

\subsection{A Definition for Belief}

We have thus far successfully augmented successor state axioms and extended the formalism for modeling noisy actions. The final question, then, is how can the outcomes of a noisy action, and their likelihoods, be accounted for? Indeed, a formula might not only be true as a result of the actions intended, but also as a result of those that were not.

Consider the simple case of deterministic actions, where 
 the density associated with \( s \) is simply  transferred to \( do(a,s) \). This is  an instance of Lewis's  \emph{imaging} \cite{lewis1976probabilities}. In contrast, if $a$ and $a'$ are $\alt$-related, then the result of doing $a$ at $s$ would lead to successor situations \( do(a,s) \) and \( do(a',s) \). Moreover, unlike noisy sensors, 
 \( a \) and \( a' \) may affect fluent values in different ways, which is certainly the case with \( \move(3,3.074) \) and \( \move(3,3) \) on the fluent \( \xpos \). Thus, the idea then is that when reasoning about the agent's beliefs about \( \phi \), one would need to integrate over the densities of all those potential successors where \( \phi \) would hold.

 To make this precise, let us first consider the result of doing a single action \( a \) at  \( \ins. \)  The \emph{degree of belief} in \( \phi \) after doing \( a \) is now an abbreviation for: \[
\dis	\bel(\phi, do(a, \ins)) \defeq  \frac{1}{\gamma} \fullx \int_z P(\vec x,  z, \phi, do(a,\ins) )
\]
where \[
	P(\vec x, z, \phi, do(a,\ins)) \doteq \probabbr{{\ivar, {b}}}{\bigwedge f_i(\ivar) = x_i \land \alt(a, b, z) \land \phi[do(b,\ivar)]}{p(do(b, \ivar), do(a, \ins))~}.
\]
(As before, the $i$ ranges over the indices of the fluents in $\L$, that is, $\{1, \ldots, n\}$.) 
The intuition is this. Recall that by integrating over \( \vec x \), all possible initial situations are considered by \( f\sub i (\ivar) = x\sub i. \) Analogously, by integrating over \(  z, \) all possible action outcomes are considered by \( \alt(a, b, z). \) Supposing \( a=A(x,y) \), for each outcome \( b = A(x,z) \),\footnote{For ease of presentation, we assume  that the nominal and the actual arguments involve a single variable.} we test whether \( \phi \) holds at the resulting situation \( do(b,\ivar) \) as before, and use its \( p \)-value. Here, this \( p \)-value is given by \( p(do(b, \ivar), do(a, \ins)), \) where the first argument is the  successor of interest \( do(b,\ivar) \) and the second is the real world \( do(a,\ins) \).

The generalization, then, for a sequence of actions is as follows: 
\begin{ldefn}\label{defn:belief continuous effector sensor} \textbf{(Degrees of belief (continuous noisy effectors and sensors.))} Suppose \( \phi \) is any \( \L \)-formula. Then the \emph{degree of belief} in \( \phi \) at \( s \), written \( \bel(\phi, s), \) is defined as an abbreviation: 
	\[ \bel(\phi,s) \doteq  \begin{array}{l}
	\dis \frac{1}{\gamma} \fullx \int_{\vec z}  P(\vec x, \vec z, \phi,s)
	\end{array}
\]
where, if \( s = do([a_1, \ldots, a_k],\ins) \), then \[ \begin{array}{l}
	P(\vec x, \vec z, \phi,s) \doteq \\ \quad\quad \lan 	\ivar, {b_1, \ldots, b_k}.~\bigwedge f_i(\ivar) = x_i \land \bigwedge \alt(a_j, b_j, z_j) \land \phi[do([b_1, \ldots, b_k], \ivar)] \\ \quad\quad\quad\quad\quad\quad\quad\quad\quad\quad\quad\quad\quad\quad\quad\quad\quad\quad\quad\quad \rightarrow p(do([b_1, \ldots, b_k],\ivar), do([a_1, \ldots, a_k], \ins)) ~\ran.
\end{array}
\]

\end{ldefn}

\noindent (Here, $i$ ranges over $\{1,\ldots, n\}$ as before, and $j$ ranges over the indices of the ground actions $\{1, \ldots, k\}$.) 
That is, given any sequence, for all possible \( \vec z \) values, we consider alternate sequences of ground action terms and integrate the densities of successor situations that satisfy \( \phi \), using the appropriate \( p \)-value.

\subsection{Example} %
\label{sub:example}

	Let us now build a simple example with noisy actions. Consider the robot scenario in Figure \ref{fig:robot}.	Suppose the basic action theory \( \D \) includes the foundational axioms, and the following components.

	The initial theory \(  \D_0 \) includes the following \( p \) specification: \begin{equation}\label{eq:p uniform 10 12}
		p(\ivar,\ins) = \begin{cases}
			.5 & \text{ if \( 10\leq \xpos(\ivar) \leq 12 \)} \\ 
			0 & \text{ otherwise}
		\end{cases}	
	\end{equation} 
For simplicity, let \( \xpos \) be the only fluent in the domain, and assume that actions are always executable. The successor state axiom for the fluent \( \xpos \) is \eqref{eq:xpos ssa noisy}. For \( p \), it is the usual one, \emph{viz.} \eqref{eq:p ssa}. 

We imagine two actions in this domain, one of which is the noisy move \( \move(x,y) \) and a sonar sensing action \( \sonar(z). \) For the alternate actions axioms, let us use \eqref{eq:alt move action} and \eqref{eq:alt sonar}. 

Finally, we specify the likelihood axioms. Let the sonar's error profile be 
\beq
l(\xsense(z),s) = u \equiv (z\geq 0\land u = \N(z-\xpos(s);0,.25)) \lor (z<0 \land u=0).
\eeq
Readers may note that this sonar is more accurate than the one characterized by \eqref{eq:sonar error model}, as it has a smaller variance. Regarding the likelihood axiom for \( \move(x,y) \),  let that be: \begin{equation}\label{eq:error profile noisy move}
	l(\move(x,y), s) = u \equiv u = \N(y-x; 0, 1).
\end{equation}

\noindent This completes the specification of \( \D. \)	\begin{theorem}\label{thm:continuous noisy effector and continuous noisy sensor} The following are entailments of \( \D \). 
	
\end{theorem}

\begin{figure}[t]
  \centering
    \includegraphics[width=.5\textwidth]{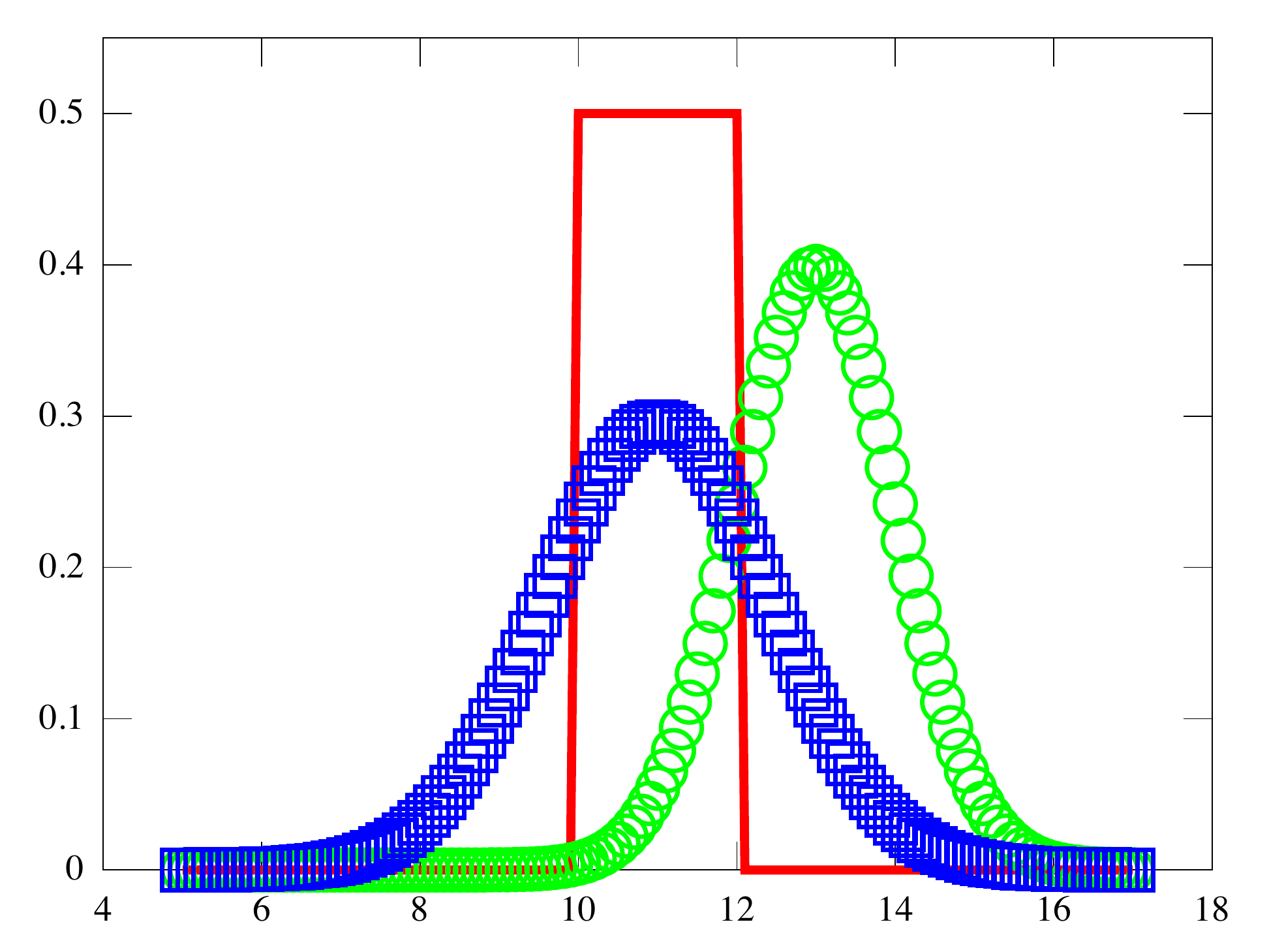}
  \caption{Belief  change for \( h \) at \( \ins \) (in solid red), a noisy move away from the wall (in green with circular markers), and after a second noisy move towards the wall (in blue with square markers).}
  \label{fig:uniform-gaussian}
\end{figure}
 \begin{enumerate}
	\item \( \bel(\xpos \geq 11, do(\move(-2,-2.01), \ins)) \approx .95  \) 
	
	We first observe that for calculating the degrees of belief, we have to consider all those successors of  initial situations wrt \( \move(-2, z) \) for every \( z \), where \( \phi \) holds. By \eqref{eq:p ssa}, the \( p \) value for such situations is the initial \( p \) value times the likelihood of \( \move(-2, z) \), which is \( \N(z+2; 0, 1) \) by \eqref{eq:error profile noisy move}. Therefore, we get \[
		\begin{array}{l}
\dis			\frac{1}{\gamma} \int_x \int_z \begin{cases}
				.5\times \N(z+2;0,1) & \text{ if \( \thereis \ivar. \) \( x \in [10,12] \), \( \xpos(\ivar) =x \) and \( (\xpos\geq 11)[do(\move(2,z),\ivar)] \) } \\ 
				0 & \text{otherwise}
			\end{cases}		\qquad \mbox{} 	\hfill = \\[3ex]
			\dis \frac{1}{\gamma} \int_x \int_z \begin{cases}
				.5\times \N(z+2;0,1) & \text{if \( \thereis \ivar. \) $ x\in [10,12]$, \( \xpos(\ivar) = x \) and \( \xpos(\ivar) - z \geq 11 \)} \\ 
				0 & \text{otherwise}
			\end{cases} \hfill  = \\[3ex]
			\dis \frac{1}{\gamma} \fullintegral \int_{10}^{12} \begin{cases}
				.5\times \N(z+2;0,1)  & \text{if \( x\in [10,12] \) and \( x-z \geq 11 \)} \\
				0 & \text{otherwise}
			\end{cases}
		\end{array} 
	\]
It is not hard to see that had the action been deterministic, the degree of belief in \( \xpos \geq 11 \) after moving away by 2 units should have been precisely 1. In Figure \ref{fig:uniform-gaussian}, we see the effect of this move, where the range of \( h \) values with non-zero densities extends considerably more than 2 units.

\item \( \bel(\xpos \geq 10, do([\move(-2,-2.01), \move(2,2.9)], \ins)) \approx .74 \)

The argument proceeds in a manner identical to the previous demonstration. The density function is further multiplied by a factor of \( \N(u-2; 0, 1) \), from \eqref{eq:error profile noisy move} and \eqref{eq:p ssa}. More precisely, we have \[
	\begin{array}{l} 
		\dis \frac{1}{\gamma} \int_x \int_z \int_u \begin{cases} .5\times\N(z+2;0,1)\times\N(u-2;0,1) & \text{if \( \thereis \ivar. \) \( x\in [10,12] \), \( \xpos(\ivar)=x \) and \( \xpos(\ivar) - z- u \geq 10 \)} \\
		0 & \textrm{otherwise}
		\end{cases} \hspace{1cm} \hfill = \\[3ex]
\dis \frac{1}{\gamma} \int_z \int_u \int_{10}^{12}  \begin{cases}
	.5\times\N(z+2;0,1)\times\N(u-2;0,1) & \text{if \( x\in [10,12] \) and \( x - z - u \geq 10 \)} \\ 
	0 & \text{otherwise}
\end{cases}	%
	\end{array} 
\]
If the action were deterministic, yet again the degree of belief about \( \xpos \geq 10 \) would be 1 after the intended actions. That is, the robot moved away by 2 units and then moved  towards the wall by another 2 units, which means that \( \xpos \)'s current value should have been precisely what the initial value was. 

See {Figure}~\ref{fig:uniform-gaussian} for the resulting density change. 
Intuitively, the resulting density changes as effectuated by the moves degrades the agent's confidence considerably. In {Figure}~\ref{fig:uniform-gaussian}, for example, we see that in contrast to a single noisy move, the range of \( h \) values considered possible has extended further, leading to a wide curve.

\item $\bel(\xpos\geq 11, do([\move(-2, -2.01), \sonar(11.5)], \ins) \approx .94 $

This demonstrates the result of a sensing action after a noisy move. Using arguments analogous to those in the previous item, it is not hard to see that  we have: \[
	\begin{array}{l}
\dis	\frac{1}{\gamma} \int_x \int_ z \begin{cases}
		.5\times \N(z+2; 0, 1) \times \N(x-z-11.5;0,.25) & \text{if \( \thereis \ivar.~x\in [10,12], \xpos(\ivar) = x \) and \( 
		\xpos(\ivar) - z \geq 11 \)} \\ 
		0 & \text{otherwise}
	\end{cases}
	\end{array}
\]

\item $\bel(\xpos\geq 11, do([\move(-2, -2.01), \sonar(11.5), \sonar(12.6)], \ins) \approx .99$ 

In this case, two successive readings around 12 strengthens the agent's belief about \( \xpos \geq 11. \) The density function is multiplied by \( 
\N(x-z-12.6;0,.25) \) because of \eqref{eq:p ssa} and \eqref{eq:sonar error model} as follows: \[
\begin{array}{l}
	\dis \frac{1}{\gamma} \int_x \int_ z \begin{cases}
		\delta \times \N(x-z-11.5;0,.25) \times \N(x-z-12.6;0,.25) & \text{if \( \thereis \ivar.~x\in [10,12], \xpos(\ivar)=x \), \( \xpos(\ivar) - z \geq 11 \)} \\ 
		0 & \text{otherwise}
	\end{cases}  \hfill = \\[3ex] 
	\dis \frac{1}{\gamma} \int_z \int_x \begin{cases} 
		\delta \times \N(x-z-11.5;0,.25) \times \N(x-z-12.6;0,.25) & \text{if \( x\in [10,12] \) and \( x - z \geq 11 \)} \\ 
		0 & \text{otherwise}
	\end{cases}
	\end{array}
\]
where \( \delta = .5\times \N(z+2;0,1) \), and 
 \( \gamma \) is \[
	\dis 	\int_z \int_x \begin{cases} 
			\delta\times \N(x-z-11.5;0,.25) \times \N(x-z-12.6;0,.25) & \text{if \( x\in [10,12] \)} \\ 
			0 & \text{otherwise}\end{cases}
\]
In Figure \ref{fig:uniform-gaussian-sensing}, the agent's increasing confidence is shown as a result of these sensing actions. Note that even though  the sensors are noisy, the agent's belief about \( 
\xpos \)'s true value sharpens because the sensor is a fairly accurate one.

\end{enumerate}

	\begin{figure}[t]
	  \centering
	    \includegraphics[width=.5\textwidth]{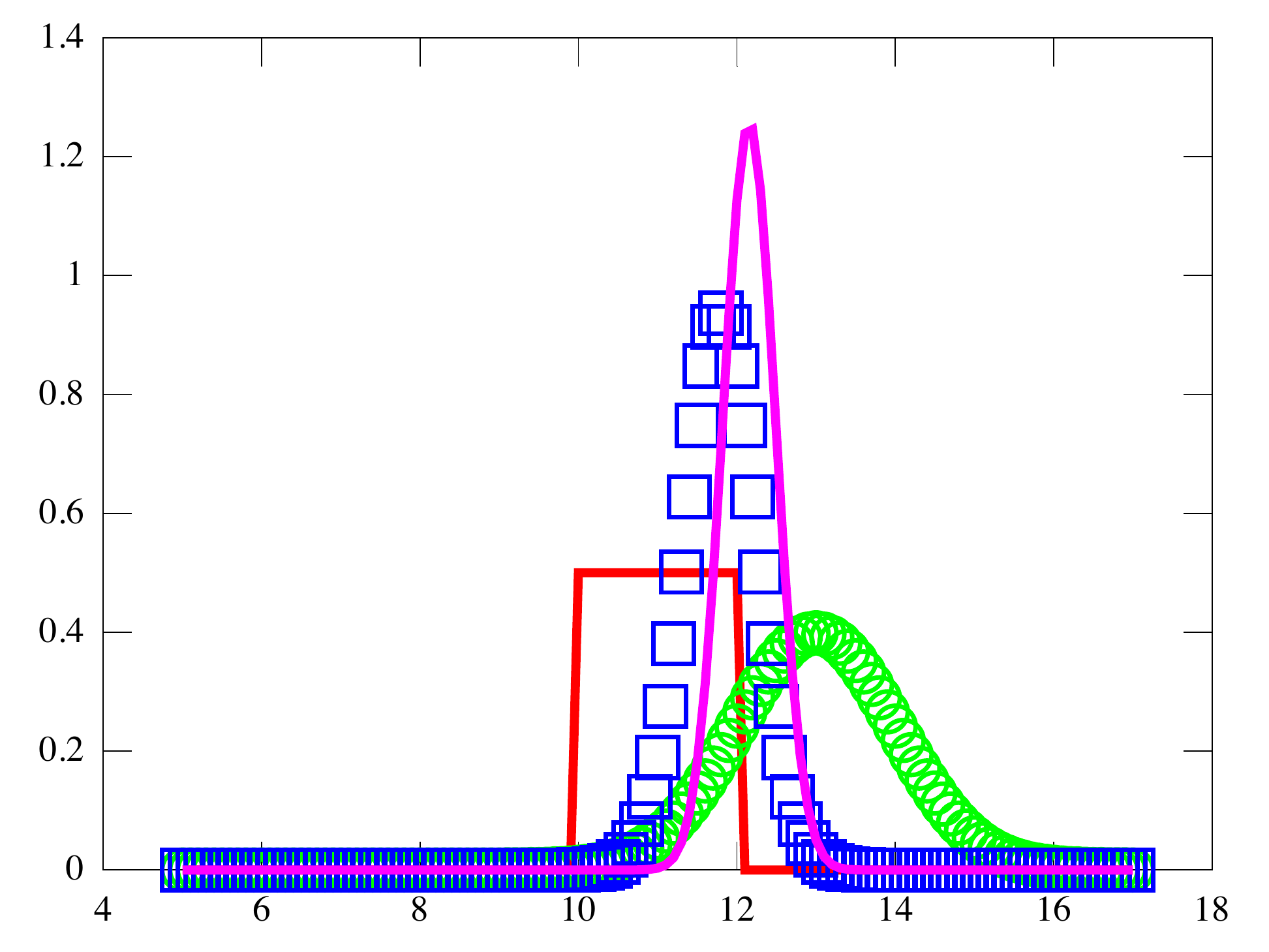}
	  \caption{Belief change for \( h \) at \( \ins \) (in solid red), after a noisy move (in green with circular markers), after sensing once (in blue with square markers), and after sensing twice (in solid magenta). }
	  \label{fig:uniform-gaussian-sensing}
	\end{figure}

\section{Generalization} %
\label{sec:generalizations}

Many real-world problems have both continuous and discrete components  (sensors, fluents, and/or effectors). Not surprisingly, discrete sensors can be easily modeled in the current scheme, as they only affect the $p$-values. Regarding fluents and effectors, it turns out that accommodating the more general case is an easy exercise, where an integration symbol in $\bel$ corresponding to a continuous fluent or action argument is replaced by a summation symbol. %

To clarify,  we proceed as follows.
  We begin with an example for discrete sensors, introduce a general definition for $\bel$ in the above sense, and finally conclude with an example that demonstrates this general setting.

\subsection{Example}

We understand a discrete sensor to mean a sensing action that is characterized by a finite number of possible observations. Thus, these observations would be associated with a probability rather than a density. Imagine the robot scenario from Figure \ref{fig:robot}. Suppose that instead of a sensor that returns a number indicating the distance to the wall, the robot is equipped with a crude binary  version. This latter sensor simply indicates whether the robot is close or far from the wall. 

Formally, suppose there is a sensing action \( \prox(z) \) where \( z \in \{ \ti{close}, \ti{far} \}. \) A noise-free model for the closeness sensor might be as follows: \[
	\begin{array}{l}
	l(\prox(z),s) = u \equiv (z = \ti{close} \land ((\xpos(s) \leq 1 \land u = 1) \lor (\xpos(s) >1 \land u = 0)) ) ~\lor \\ 
	\qquad \mbox{} \qquad \mbox{} \quad ~~\qquad \mbox{} \qquad \mbox{} (z = \ti{far} \land ((\xpos(s) > 1 \land u = 1) \lor (\xpos(s) \leq 1 \land u = 0)) ).
	\end{array}
\]
For the more interesting case of a noisy sensor, assume the following error profile: 
	\begin{equation}\label{eq:close sensor}
		\begin{array}{l} 
			l(\prox(z),s) = u \equiv (z = \close \land \xpos(s) \leq 3 \land u = 2/3) ~\lor \\
		\qquad \mbox{} \qquad \mbox{} \qquad \mbox{} \qquad \mbox{} \quad~~~ 	(z = \close \land \xpos(s) > 3 \land u = 1/3) ~\lor \\ 
		\qquad \mbox{} \qquad \mbox{} \qquad \mbox{} \qquad \mbox{} \quad~~~	(z = \far \land \xpos(s) \geq 4 \land u = 4/5) ~\lor \\ 
	\qquad \mbox{} \qquad \mbox{} \qquad \mbox{} \qquad \mbox{} \quad~~~		(z = \far \land \xpos(s)<4 \land u= 1/5). 
		\end{array}
	\end{equation}
Notice that the behavior of the sensor differs for the two values. So, for \( z = \close \) the sensor considers \( \xpos \leq 3 \) as a measure of closeness, and has an accuracy of 2/3. For \( z = \far, \) however, the sensor takes  \( \xpos \geq 4 \) to be a measure of being far, and has an accuracy \( 4/5. \)  

We now build a simple action theory using this sensor. For simplicity,  assume noise-free physical actions and a single fluent, \( \xpos \). 
 Let \( \D \) be the union of the $p$ specification \eqref{eq:uniform 2 12 continuous}, the successor state axiom  \eqref{eq:xpos ssa},  the above likelihood model \eqref{eq:close sensor}, in addition to \eqref{eq:p initial constraint}, \eqref{eq:p ssa} and \eqref{eq:stara}. For the actions \( A(z) \) in \( \D \), including the sensor, let \[
	\alt(A(z), a', u) \equiv a' = A(u) \land u =z. 
\]
We now analyze belief change on the application of this sensor. 
\begin{theorem} The following are entailments of \( \D \):

\end{theorem} 

\begin{enumerate}
	\item \( \bel(\xpos \leq 4, do(\prox(\close),\ins)) = 3/11  \)
	
	By \eqref{eq:p ssa}, after \( a = \prox(\close) \) the term \( p(do(a,s'), do(a,s))  \) is obtained from \( p(s', s) \times \Err(\close,\xpos(s')) \), where \( \Err(u_1, u_2) \) denotes the error profile \eqref{eq:close sensor} of the closeness sensor. To compute the belief term, let us first resolve the normalization factor \( \gamma. \) 
	It is not hard to see that  \( \gamma \) evaluates to \[
		\begin{array}{l}
\dis			\int_x  \begin{cases}
.1\times 2/3 & \text{if \( \thereis \ivar.~\xpos(\ivar) = x, x \in [2,12]  \) and \( \xpos(\ivar) \leq 3 \)} \\ 
.1 \times 1/3 & \text{if \( \thereis \ivar.~\xpos(\ivar) = x, x \in [2,12]  \) and \( \xpos(\ivar) > 3 \)} \\
0 & \text{otherwise}
\end{cases}
		\end{array}
	\]
Since \( \D_0 \) assigns a non-zero density to only those situations where \( \xpos \in [2,12] \) one obtains the above normalization factor.  Moreover, after doing \( \prox(\close) \), the density of those situations where \( \xpos\leq 3 \) is multiplied by a factor of 2/3, while the density of the remaining situations is multiplied by a factor of 1/3. 

Formulating the numerator is  analogous,  except that only those situations where \( \xpos \leq 4 \) are to be considered. To be precise, the degree of belief in \( \xpos \leq 4 \) after the sensing action is: \[
	\begin{array}{l}
		\dis \frac{1}{\gamma} \int_x \begin{cases} 
		.1 \times 2/3 & \text{if \( x\in [2,12], x\leq 3, x \leq 4 \)} \\ 
		.1 \times 1/3 & \text{if \( x\in [2,12], x> 3, x \leq 4 \)} \\
		0 & \text{otherwise}
		\end{cases} \hspace{3cm} \hfill = \\[5ex]
\dis \frac{1}{\gamma} 	\int_x \begin{cases} 
		.1 \times 2/3 & \text{if \( x\in [2,3] \)} \\ 
		.1 \times 1/3 & \text{if \( x\in (3,4] \)} \\
		0 & \text{otherwise}
\end{cases}	
	\end{array}
\]
This then amounts to 3/11. 
\item \( \bel(\xpos\leq 4, do([\move(1), \prox(\close)])) = 5/12 \)	

We proceed in  a manner analogous to the previous item. Using \eqref{eq:p ssa} and \eqref{eq:close sensor}, we obtain: \[
	\begin{array}{l}
		\dis \frac{1}{\gamma} \int_x \begin{cases}
			.1 \times 2/3 & \text{if \( \thereis \ivar.~\xpos(\ivar) = x, \xpos(do(\move(1),\ivar)) \leq 3 \) and \( \xpos(do(\move(1),\ivar)) \leq 4 \)} \\
			.1\times 1/3 & \text{if \( \thereis \ivar.~\xpos(\ivar) = x, \xpos(do(\move(1),\ivar)) > 3 \) and \( \xpos(do(\move(1),\ivar)) \leq 4 \)} \\
			0 & \text{otherwise} 
		\end{cases} \hspace{1cm} = \\[5ex] 
	\dis \frac{1}{\gamma} \int_x \begin{cases}
		.1 \times 2/3 & \text{if \( x\in [2,12], x -1 \leq 3 \) and \(  x - 1\leq 4 \)} \\
			.1\times 1/3 & \text{if \( x\in [2,12], x-1>3 \) and \( x-1\leq 4 \)} \\
			0 & \text{otherwise}
	\end{cases}	
	\end{array}
\]
The numerator amounts to 5/3 and the normalization factor \( \gamma  \) is 4, leading to a degree of belief of 5/12. 
\end{enumerate}

\subsection{A General Definition for Belief}

The main idea is to simply allow the range of some fluents to be taken from finite sets. So, reconsider Definition \ref{defn:continuous belief}, where the fluents and the action arguments were assumed to take values from $\real$. Then, suppose fluents $f\sub 1, \ldots, f\sub n$ takes values from $\real$, and fluents $g\sub 1, \ldots, g\sub m$ take values from finite sets. Intuitively, $f\sub i$ is to be seen as a continuous probabilistic variable, and $g\sub {\bar i}$  as  a discrete probabilistic variable.\footnote{Of course, discrete probabilistic variables can also take values from infinite sets, and by way  of Section \ref{sub:integration}, limits can be used to define the sum of an infinite sequence of terms, that is, a series. For simplicity of presentation, however, we assume discrete fluents and action arguments   take values from finite sets.} Analogously, suppose $a\sub 1, \ldots, a\sub o$ are action types such that $\alt(a\sub j, b, z)$ holds for $z\in \real$, and suppose $d\sub {1}, \ldots, d\sub l$ are action types such that $\alt(d\sub {\bar j}, r, u)$ holds for $u$ taken from a finite set. 
Intuitively, $a\sub j$ is to be seen as a noisy action characterized by a continuous probability distribution, and $d\sub {\bar j}$ as a noisy action characterized by a discrete probability distribution. Then:

\begin{ldefn}\label{defn:continuous and discrete fluents belief} {\bf (Degrees of belief (discrete and continuous fluents, sensors and effectors)).} Suppose \( \phi \) is any \( \L \)-formula. The {\it degree of belief} in $\phi$ at $s$, written \( \bel(\phi,s) \), is an abbreviation:  \[ \begin{array}{l}
\bel(\phi,s) \doteq	\dis \frac{1}{\gamma} \int_{\vec x} \sum \raisebox{-7pt}{${\!}_{\vec{y}}$} \int_{\vec z}  \sum \raisebox{-7pt}{${\!}_{\vec{u}}$} P(\vec x \cdot \vec y, \vec z \cdot \vec u, \phi, s)
		\end{array}
\]
where if $s=do(\alpha,\ins)$ for \( \alpha = [a_1, \ldots, d\sub 1, \ldots, a_o, \ldots, d_l] \), then \[ \begin{array}{l}
	P(\vec x \cdot \vec y, \vec z \cdot \vec u, \phi, s) \doteq \\ 
	\quad\quad \lan 	\ivar, {b_1, \ldots, b_o, r_1, \ldots, r_l }.~\bigwedge f_i(\ivar) = x_i \land \bigwedge g_{\bar i}(\ivar) = y_{\bar i} ~\land \\
	\quad\quad\quad\quad \quad\quad\quad\quad \quad\quad\quad \bigwedge \alt(a_j, b_j, z_j) \land \bigwedge \alt(d_{\bar j}, r_{\bar j}, u_{\bar j}) \land \phi[do([b_1, \ldots, r_l], \ivar)]  \\ \quad\quad\quad\quad \quad\quad\quad\quad \quad\quad\quad\quad \quad\quad\quad\quad \quad\quad\quad\quad \quad\quad\quad\quad  \rightarrow p(do([b_1, \ldots, r_l],\ivar), do([a_1, \ldots, d_l], \ins))\, \ran.
\end{array}
\] 
Here, $\cdot$ is for the concatenation of variables, $i$ ranges over the indices of the continuous fluents, $\bar i$ over the indices of the discrete fluents, $j$ over the indices of the continuous actions, and $\bar j$  over the indices of the discrete actions.  
\end{ldefn}

Naturally, by means of \eqref{eq:stara}, there is an initial situation for every possible real number for \( f_1, \ldots, f_n \) and for every possible vector of values for \( g\sub 1, \ldots, g_m. \) 
The intuition is as before. That is, owing to a bijection between situations and the vector of fluent values, for any given value for \( x_1, \ldots, x_n, y\sub 1, \ldots, y\sub m \), there is a unique initial situation  where \( f_i \) has the value \( x_i \) and $g\sub {\bar i}$ has the value $y\sub {\bar i}$. The only difference to what we had earlier is that instead of just integrating over possible values for \( \vec x \), of course, we integrate over values for \( \vec x \) and sum over possible values for \( \vec y \) while using the \( p \)-values of successor situations where \( \phi \) holds. When a noisy action occurs, the space of possible successor situations is determined by $\alt(a_j,b,z)$ for a noisy action with a continuous probabilistic model, and the space is determined by $\alt(d_{\bar j},r,u)$ otherwise.

\subsection{Example} %
\label{sub:scheme_e_and_scheme_f}

To demonstrate the more intricate case of discrete and continuous fluents, and discrete and continuous action arguments, we consider an example of a robot moving towards the wall with a window, as shown in Figure \ref{fig:robotwindow}. Like with Figure \ref{fig:robot}, let a fluent \( h \) denote the distance to the wall. 
Let \( \win \) be a  fluent that captures the status of the window in the sense of whether it is opened or closed,  with \( \win = 1 \) meaning that it is open and \( \win = 0 \) meaning that it is closed.

To change the status of the window, we imagine a noisy effector \( \chw(x,y) \) where the agent intends on setting \( \win \) to \( x \), but it is, in fact, set to \( y. \) We formally account for this effector using: \begin{equation}\label{eq:alt chw}
	\alt(\chw(x,y), a', z, s) \equiv a' = \chw(x,z)
\end{equation}
\begin{equation}\label{eq:likelihood chw}
	\begin{array}{l}
	l(\chw(x,y),s) = u \equiv ((x = 0 \lor x = 1) \land x=y \land u = .75) ~\lor \\ 
	\qquad \mbox{} \qquad \mbox{} \qquad \mbox{} \qquad \mbox{} ~~~~~((x = 0 \lor x = 1) \land (y = 0 \lor y = 1) \land |x-y| = 1 \land u = .25) ~\lor \\ 
	\qquad \mbox{} \qquad \mbox{} \qquad \mbox{} \qquad \mbox{} ~~~~~	( (\neg (x = 0 \lor x = 1) \lor \neg(y = 0 \lor y = 1) ) \land u = 0). 
	\end{array}
\end{equation}
This can be seen as saying that \( \set{x,y} \) take values \( \set{0,1} \), and the likelihood of \( x \) and \( y \) agreeing is .75, and that of  disagreeing is .25. So, it is very likely that the action succeeds. 

From this, we then provide the following successor state axiom for \( \win \): \begin{equation}\label{eq:ssa win}
	\begin{array}{l}
	\win(do(a,s)) = u \equiv \thereis x,y[\chw(x,y) \land u = y] ~\lor \\ 
\qquad \mbox{} \qquad \mbox{} \qquad \mbox{}~~~~~	\neg\thereis x,y [\chw(x,y) \land u = \win(s)].
	\end{array}
\end{equation}
For simplicity, assume that the action \( \xmove(z) \) moves the robot towards the wall by \( z \) units (that is, it is deterministic), for which we use the  successor state axiom \eqref{eq:xpos ssa}.

Analogously, imagine a noisy sensor \( \seew(z) \) that provides a reading of \( x \) for the status of the window, but with the following error profile: \begin{equation}\label{eq:likelihood seew}
	\begin{array}{l}
	l(\seew(z),s) = u \equiv (z = 1 \land \win(s) = 1 \land u = .8) ~\lor \\
\hfill	(z = 1 \land \win(s) = 0 \land u = .2) ~\lor \\
\hfill	(z = 0 \land \win(s) = 1 \land u = .3) ~\lor \\
\hfill	(z = 0 \land \win(s) = 0 \land u = .7). ~~~
	\end{array}
\end{equation}
That is, when returning 1, the sensor gives the correct reading with a probability of .8, but when returning 0, the sensor gives the correct reading with a probability of .7. Alternate action axioms are specified as usual for sensors: \begin{equation}\label{eq:alt seew}
	\alt(\seew(x), a', z, s) \equiv a' = \seew(z) \land z=x. 
\end{equation}

\begin{figure}[ht]
	\centering
		\includegraphics[height=3.5cm]{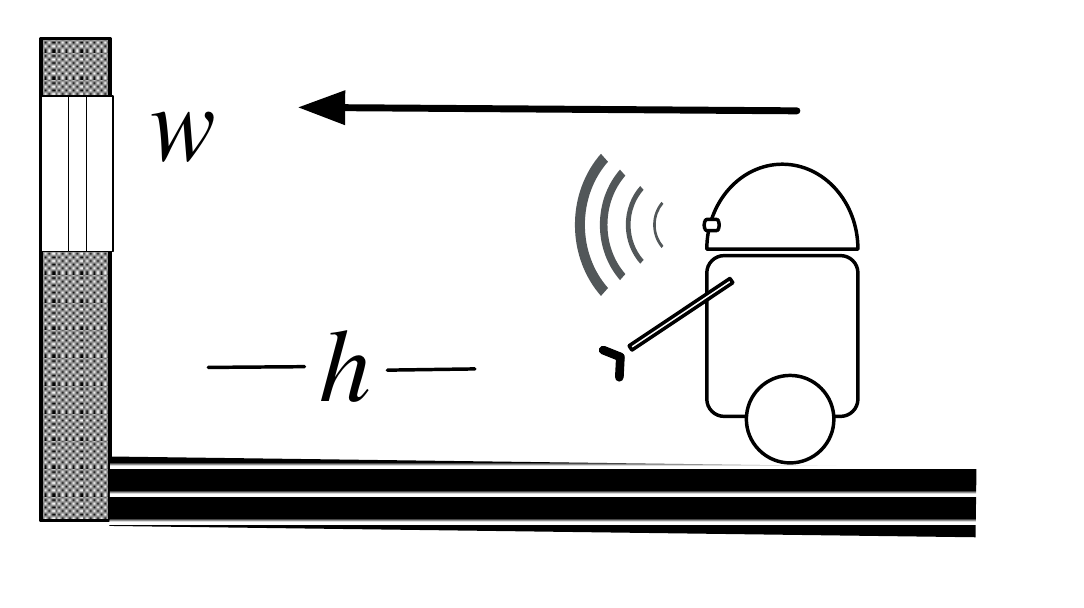}
	\caption{Robot moving towards a wall that has a window.}
	\label{fig:robotwindow}
\end{figure} 

To finalize the  example, 
 let \( \D \) be a union of \eqref{eq:xpos ssa},  \eqref{eq:alt chw}, \eqref{eq:likelihood chw}, \eqref{eq:ssa win}, \eqref{eq:alt seew}, \eqref{eq:likelihood seew}, together with \eqref{eq:p initial constraint}, \eqref{eq:p ssa}, \eqref{eq:stara}, and the following initial axiom for \( p \): \begin{equation}\label{eq:p initial spec with window fluent}
\begin{array}{l}
p(\ivar, \ins) = \begin{cases}
.5 \times .6 & \text{if \( 10\leq \xpos(\ivar) \leq 12 \) and \( \win(\ivar)  = 1 \)} \\
.5 \times .4 & \text{if \( 10 \leq \xpos(\ivar) \leq 12 \) and \( \win(\ivar) = 0 \)} \\ 
0 & \text{otherwise}
\end{cases}
\end{array}
\end{equation}
That is, \( \xpos \) and \( \win \) are independent, \( \xpos \) is uniformly distributed on \( [10,12] \) and the window being open has a probability of .6. 

\begin{theorem} The following are entailed by \( \D \): 

\end{theorem} \begin{enumerate}
	\item \( \bel(\win = 0, \ins) = .4 \)
	
	We are to compute the  following term, which is easily shown to be equal to .4: \[
		\dis \frac{1}{\gamma} \int_x \prettysum y \begin{cases}
		.5 \times .6 & \text{$\thereis \ivar.~x \in [10,12], \xpos(\ivar) = x,  \win(\ivar) = y, $ \( y = 1 \) and \( \win(\ivar) = 0 \)} \\ 
			.5 \times .4 & \text{$\thereis \ivar.~x \in [10,12], \xpos(\ivar) = x,  \win(\ivar) = y, $ \( y = 0 \) and \( \win(\ivar) = 0 \)} 	 \\ 
			0 & \text{otherwise}
		\end{cases}
	\]
	
	\item \( \bel(\win = 0, do(\move(1),\ins)) = .4 \)
	
	Owing to Reiter's solution to the frame problem, belief in the closed window does not change on moving laterally. If we were to expand the belief term, we could get a logical term equivalent to the one in the previous item. 
	
	\item \( \bel(\win=0, do(\chw(0,1),\ins))  = .75 \)
	
	After attempting to close the window, the agent's belief about the window being closed is .75. Not surprisingly, since the action sets the final value of \( \win \) (as opposed to toggle its value), the degree of belief precisely corresponds to the probability of the action getting successfully executed in \eqref{eq:likelihood chw}. Expanding the belief term,  we first obtain: \[
		\begin{array}{l}
		\dis \frac{1}{\gamma} \int_x \prettysum y \prettysum z \begin{cases}
		.5 \times .6 \times \delta(0,z) & \text{\( \thereis \ivar.~\xpos(\ivar) = x, x\in [10,12], \win(\ivar) = y, y = 1, (\win = 0)[do(\chw(0,z),\ivar)] \)} \\ 
	.5 \times .4 \times \delta(0,z) & \text{\( \thereis \ivar.~\xpos(\ivar) = x, x\in [10,12], \win(\ivar) = y, y = 0, (\win = 0)[do(\chw(0,z),\ivar)] \)} \\ 
	0 & \text{otherwise}		
		\end{cases}
		\end{array}
	\]
Here \( \delta(0,z) \)	is the probability assigned to the substitution of \( z \) in \( \chw(0,z) \) via \eqref{eq:likelihood chw}. 
Simplifying the above, we get: \[
\begin{array}{l}
		\dis \frac{1}{\gamma} \int_x \prettysum y \prettysum z  \begin{cases}
			.5 \times .6 \times .75 & \text{if \( \thereis \ivar.~\xpos(\ivar) = x, x\in [10,12], \win(\ivar) = y, y = 1, z = 0, \) \( (\win = 0)[do(\chw(0,z),\ivar)] \)} \\ 
			.5 \times .4 \times .75 & \text{if \( \thereis \ivar.~\xpos(\ivar) = x, x\in [10,12], \win(\ivar) = y, y = 0, z = 0, \) \( (\win = 0)[do(\chw(0,z),\ivar)] \)} \\
			0 & \text{otherwise} 
		\end{cases} \hspace{.3cm} =  \\[5ex]
	\dis \frac{1}{\gamma} \int_x \prettysum y \prettysum z \begin{cases}
	.5 \times .6 \times .75 & \text{if \( x \in [10,12], y = 1, z = 0 \)} \\ 
	.5 \times .4 \times .75 & \text{if \( x\in [10,12], y= 0, z = 0 \)} \\ 
	0 & \text{otherwise}
	\end{cases}
\end{array}
\] Essentially, there are only two ways that the window can be closed after doing \( \chw(0,z). \) Either the window is closed initially and \( z = 0 \), or the window is open and again \( z = 0. \) This leads to .75. 
\item \( \bel(\win = 0, do([\chw(0,1), \seew(0)], \ins)) = .875 \)

After observing that the window is closed from its sensor, the robot's belief about the window being closed increases. Proceeding in a manner analogous to above, after simplifications, we get: \[
	\dis \frac{1}{\gamma} \int_{x} \prettysum y \prettysum z \begin{cases}
	.5 \times .6 \times .75 \times .7 & \text{if \( \thereis \ivar.~\xpos(\ivar) = x, x\in [10,12], \win(\ivar) = y, y = 1, z = 0, (\win = 0)[do(\alpha,\ivar)] \)} \\ 
	.5 \times .4 \times .75 \times .7 & \text{if \( \thereis \ivar.~\xpos(\ivar) = x, x\in [10,12], \win(\ivar) = y, y = 0, z = 0, (\win = 0)[do(\alpha,\ivar)] \)} \\
	0 & \text{otherwise}
	\end{cases}
\]
where \( \alpha = [\chw(0, z), \seew(0)]. \) The numerator simplifies to \( .75 \times .7. \) It can be shown that \( \gamma \) is: \[
	\dis  \int_{x} \prettysum y \prettysum z \begin{cases}
	.5 \times .6 \times .75 \times .7 & \text{if \( \thereis \ivar.~\xpos(\ivar) = x, x\in [10,12], \win(\ivar) = y, y = 1, z = 0, (\true)[do(\alpha,\ivar)] \)} \\ 
	.5 \times .4 \times .75 \times .7 & \text{if \( \thereis \ivar.~\xpos(\ivar) = x, x\in [10,12], \win(\ivar) = y, y = 0, z = 0, (\true)[do(\alpha,\ivar)] \)} \\
	.5 \times .6 \times .25 \times .3 & \text{if \( \thereis \ivar.~\xpos(\ivar) = x, x\in [10,12], \win(\ivar) = y, y = 1, z = 1, (\true)[do(\alpha,\ivar)] \)} \\
	.5 \times .4 \times .25 \times .3 & \text{if \( \thereis \ivar.~\xpos(\ivar) = x, x\in [10,12], \win(\ivar) = y, y = 0, z = 1, (\true)[do(\alpha,\ivar)] \)} \\
	0 & \text{otherwise}
	\end{cases}
\]
\end{enumerate}
which leads to \( .75 \times .7 + .25\times .3. \) Thus,  the belief in  the window being closed is .875.

\subsection{Summary} %
\label{sub:summary}

This completes the specification for reasoning about beliefs with discrete and continuous fluents against noisy effectors and sensors. To summarize, the generalization of the BHL scheme to arbitrary domains is defined using convenient abbreviations for $\bel$, sums and integrals, and where an action theory consists of: \begin{itemize}
\item $\D_0$ describing the initial state, which also includes \eqref{eq:p initial constraint} and \eqref{eq:stara};
\item successor state axioms as before, including a fixed one for \( p \), namely \eqref{eq:p ssa};
\item alternate actions axioms for noisy effectors;
\item likelihood axioms for noisy sensors and noisy effectors; and
\item precondition axioms and foundational axioms as before.
\end{itemize}

It is perhaps also worth noting that our specification extends  BHL in a minimal way. They do not need  a fixed space of situations, which we do by \eqref{eq:stara}, but this is very reasonable \cite{DBLP:journals/etai/LevesquePR98}.
Likelihood axioms are specified for us in much the same manner as they would. Their treatment of noisy actions and sensors is slightly different in expecting the modeler to provide \emph{indistinguishability} axioms. Briefly, these axioms speculate the set of actions that are observationally indistinguishable to the agent. Roughly, then, this serves the same purpose as our alternate actions axioms. However, their approach requires the successor state axiom for \( p \) to include notions of indistinguishability. But perhaps most significantly, as we mentioned earlier, their approach needs to appeal to \golog~to reason about belief change after noisy effectors, which we do not. \smallskip 

To conclude our technical treatment, let us attempt to simplify Definition \ref{defn:continuous and discrete fluents belief}. We will first abuse notation and use a single symbol to denote either sums or integrals, with the understanding that they expand appropriately for a given term \( t. \) For the sake of the discussion, let us use \( \int\sub x t \) to mean the integration of the function \( t(x) \) from negative to positive infinity when the function takes values from \( \real \), and the sum of terms otherwise. Under this notational convention, let us suppose \( f\sub 1, \ldots, f\sub n \) are all the fluents in the language, some of which take values from \( \real \) and others take values from finite sets. Suppose \( a\sub 1, \ldots, a\sub l \) are all the action types in the language such that for  some action types \( \alt(a, b, z) \) holds for \( z\in \real \), and for the remaining action types, \( \alt(a, b, z)  \) holds for \( z \) taking values from finite sets.  Then, we obtain a proposal like Definition \ref{defn:continuous belief}: 

\begin{ldefn}\label{defn:belief continuous effector sensor} \textbf{(Degrees of belief (simplified and general).)} Suppose \( \phi \) is any \( \L \)-formula, and let \( \int\sub x t \) denote integration or summation over term \( t \) as appropriate. Then the \emph{degree of belief} in \( \phi \) at \( s \), written \( \bel(\phi, s), \) is defined as an abbreviation: 
	\[ \bel(\phi,s) \doteq  \begin{array}{l}
	\dis \frac{1}{\gamma}  \int_{\vec x \,\cdot\, \vec z}  P(\vec x, \vec z, \phi,s)
	\end{array}
\]
where, if \( s = do(\alpha,\ins) \) for \( \alpha = [a_1, \ldots, a_k] \) and suppose \( \beta = [b\sub 1, \ldots, b\sub k] \), then \[ \begin{array}{l}
	P(\vec x, \vec z, \phi,s) \doteq  \lan 	\ivar, b\sub 1, \ldots, b\sub k.~\bigwedge f_i(\ivar) = x_i \land \bigwedge \alt(a_j, b_j, z_j) \land \phi[do(\beta, \ivar)]  \rightarrow p(do(\beta,\ivar), do(\alpha, \ins)) ~\ran.
\end{array}
\]
That is, $i$ ranges over the indices of the fluents, and $j$ over the indices of the ground action terms.

\end{ldefn}

\section{Related Work} %
\label{sec:related_work}

This article focused on degrees of belief in a first-order dynamic setting. In particular, the existing scheme of BHL was generalized to handle both discrete and continuous probability distributions while retaining all the advantages. Related efforts on belief update via sensor information can be broadly classified into two camps: the literature on probabilistic formalisms, and those that extend logical languages. We discuss them in turn. At the outset, we remark that the focus of our work is on developing a general framework, and not on computational considerations, efficiency or otherwise.

From the perspective of probabilistic  modeling, graphical models \cite{books/daglib/0023091}, such as 
Bayesian networks  \cite{pearl1988probabilistic}, are important formalisms for dealing with probabilistic uncertainty in general, and the uncertainty that would arise from noisy sensors in particular. Mainly, when random variables are defined by a set of dependencies, the density function can be compactly factorized using these formalisms. The significance of such formalisms is computational, with reasoning methods, such as filtering, being a fundamental component of contemporary robotics and machine learning technologies  \cite{dean1991planning,fox2003bayesian,thrun2005probabilistic}. On the representation side, however, these formalisms have difficulties handling strict uncertainty, as would arise from connectives such as disjunctions. (Proposals such as credal networks \cite{Cozman2000199} allow for certain types of partial specifications, but still do not offer the generality of arbitrary logical constraints.) Moreover, since rich models of actions are rarely incorporated, shifting conditional dependencies and distributions are hard to address in a general way. While there are graphical model frameworks with an account of actions, 
such as  \cite{darwiche1994action,DBLP:conf/kr/HajishirziA10}, 
they too have difficulties handling strict uncertainty and
quantification. To the best of our knowledge, no existing probabilistic formalism handles changes in state variables like those considered here.

This inherent limitation of probabilistic formalisms led to a number of influential proposals on combining logical and probabilistic specifications \cite{nilsson1986probabilistic}. (The synthesis of deductive reasoning and the probability calculus has a long and distinguished history \cite{nla.cat-vn1006072,Gaifman19641} that we do not review here; see \cite{citeulike:115039} and references therein.) The works of Bacchus and Halpern 
\cite{174658,bacchus1990representing}, for example, provide  the means to specify properties about the domain together with probabilities about propositions; see \cite{Ognjanovic2000191} for a recent list on first-order accounts of probability. But these do not explicitly address reasoning about actions. As we noted, treating actions in a general way requires, among other things, addressing the frame problem, reasoning about what happened in the past and projecting the future,  handling contextual effects, as well as appropriate semantical machinery. We piggybacked on the powerful situation calculus framework, and extended that theory for reasoning about continuous uncertainty. 

In a similar vein, from a modal logical  perspective, the interaction between categorical knowledge, on the one hand, and degrees of belief, on the other, is further discussed in \cite{174658,RePEc:eee:gamebe:v:1:y:1989:i:2:p:170-190,Heifetz200131}. While these are essentially propositional, there are first-order variants \cite{Belleoknowprob}. Actions are not explicitly addressed, however.

Recently in AI, limited versions of probabilistic logics have been discussed, in the form of relational graphical models, Markov logic networks, probabilistic databases and probabilistic programming \cite{DBLP:conf/aaai/KollerP98,DBLP:conf/icml/GetoorFKT01,DBLP:journals/cacm/Russell15,DBLP:conf/ijcai/Poole03,DBLP:conf/aaai/DomingosW12,DBLP:conf/ijcai/RaedtKT07,suciu2011probabilistic,DBLP:conf/uai/SinglaD07,DBLP:conf/ijcai/MilchMRSOK05}.
Some have been further extended for 
 continuous probability distributions and temporal reasoning \cite{Choi:2011fk,DBLP:conf/iros/NittiLR13,anderson2002relational,Belle:2015af}. Overall, these limit the first-order expressiveness of the language, do not treat actions in a general way, and do not handle strict uncertainty. Admittedly, syntactical restrictions in these frameworks are by design, in the interest of tractability (or at least decidability) wrt inference, as they have origins in the richer probabilistic logical languages mentioned above \cite{174658,bacchus1990representing}. From the point of view of a general-purpose representation language, however, they are lacking in the kinds of features that we emphasise here.

From the perspective of dynamical systems, closest in spirit to our work here are knowledge representation languages for reasoning about action and knowledge, which we refer to as action logics. The situation calculus \cite{McCarthy:69,reiter2001knowledge}, which has been the sole focus of this paper, is one such language. There are others, of course, such as the event calculus \cite{Kowalski:1986:LCE:10030.10034}, dynamic logic  \cite{Ditmarsch:2007:DEL:1535423,DBLP:journals/logcom/DitmarschHL11},  the fluent calculus \cite{opac-b1121577}, and formalisms based on the stable model semantics \cite{gelfond1993representing}.

In the situation calculus, a monotonic solution to the frame problem was provided by Reiter \cite{Reiter:91}. The situation calculus was extended to reason about knowledge whilst incorporating this solution in \cite{citeulike:528170}, and to reason about noisy effectors and sensors by BHL.\footnote{Throughout, we operated under the setting of knowledge expansion, that is, observations are assumed to  resolve the agent's uncertainty and never contradict what is believed. The topic of \textit{belief revision} lifts this assumption  \cite{citeulike:4115507}, but it is not considered here. See \cite{DBLP:journals/ai/ShapiroPLL11} for an account of belief revision in the situation calculus.}
Other action logics have enjoyed similar extensions. For example, \cite{kooi2003probabilistic}  proposes  an extension to dynamic logic for reasoning about degrees of belief and noisy actions,  and \cite{DBLP:journals/tplp/BaralGR09} provides  a computational framework for probabilistic reasoning using the stable model semantics, 
but they are propositional. In \cite{thielscher:AI01,thielscher:ICAART10}, the fluent calculus was extended for probabilistic beliefs and noisy actions. None of these admit continuous probability distributions.

The situation calculus has also been extended for uncertainty modeling in other directions. For example, \cite{Mateus:2001:PSC:590425.590452} consider discrete noisy actions over complete knowledge, that is, no degrees of belief. In later work, \cite{DBLP:conf/aips/FritzM09} treat continuous random variables as meta-linguistic functions, and so their semantics is not provided in the language of the situation calculus. This seems sufficient for representing things like products of probabilistic densities, but it is not an epistemic account in a logical or probabilistic sense. A final prominent extension to the situation calculus for uncertainty is the embedding of decision-theoretic planning in \golog~\cite{boutilier2000decision,boutilier2001symbolic}. Here, actions are allowed to be nondeterministic, but the assumption is that the actual state of the world is fully observable. (It essentially corresponds to a  fully observable Markov decision process \cite{boutilier1999decision}.) In this sense, the picture is a special case of the BHL framework. It is also not developed as a model of belief. 
While this line of work has been extended to a partially observable setting  \cite{Sanner:2010fk}, the latter  extension is also not developed as a model of belief. Perhaps most significantly, neither of these support continuous probability distributions, nor strict uncertainty at the level of probabilistic beliefs. 

It is  worth noting that {real-valued fluents in action logics turn out to be  useful for modeling resources and time. See, for example, \cite{DBLP:journals/igpl/GrosskreutzL03,DBLP:conf/aaai/HerrmannT96,Fox:2006ve}. These are, in a sense, complementary to an account of belief.}

Outside the logical literature, there are a variety of formalisms for modeling noisy dynamical systems. Of these, partially observable Markov decision processes (over discrete and continuous random variables) are perhaps the most dominant  \cite{thrun2005probabilistic}. They can be seen as belonging to the literature on  probabilistic planning languages \cite{kushmerick1995algorithm,Younes:2004bh}. 
Recent probabilistic planning languages \cite{sannerpddl}, moreover,  combine continuous  Bayesian networks and classical planning languages. Planning languages, generally speaking, only admit a limited set of logical connectives, constrain the language for specifying dynamic laws (that is, they limit the syntax of the successor state axioms), and do not handle strict uncertainty.

In sum, to our knowledge, our proposal is the first of its kind to handle  degrees of belief, noisy sensors and effectors over discrete and continuous probability distributions in a general way. The proposal allows for partial and incomplete specifications, and the  properties of belief will then follow at a corresponding level of specificity. Moreover, to our knowledge, no other logical formalism for uncertainty deals with the integration of continuous variables within the language.

\section{Conclusions} %
\label{sec:conclusions}
Many real-world applications, such as  robotics, have to deal with numerous sources of uncertainty, the
main culprit being sensor noise. Probabilistic error models have proven to be
essential in state estimation, allowing the beliefs of a robot to be
strengthened over time.  But to use these models, the modeler is left with the
difficult task of deciding how the domain is to be captured in terms of random
variables, and shifting conditional independences and distributions. In
the BHL model, one simply provides a specification of some initial beliefs,
characterizes the physical laws of the domain, and suitable posterior beliefs
are entailed. The applicability of BHL  
was limited, however, by its inability to handle continuous distributions, a
limitation we lift in this article. By recasting the assessment of belief in
terms of fluent values, we now seamlessly combine the situation calculus with
discrete probability distributions, densities and difficult combinations of the two. We
demonstrated that distributions evolve appropriately after actions, emerging 
as a side-effect of the general
specification. Our formal framework was then shown to easily accommodate the interaction between discrete  and continuous fluents,  discrete and continuous noise models, and logical connectives. At a specification level, the framework provides the necessary bridge between logic-based action formalisms and probabilistic ones.

Armed with this general specification language, we are in a position to  investigate specialized reasoning machinery. To give a few examples, in \cite{BelleLevreg,blprog}, we identified general projection techniques, where we transform a property of belief after a sequence of (noisy) actions and observations to what is believed initially. In later work \cite{blprego}, we provided an efficient implementation of a projection technique 
under some reasonable assumptions, in service of enabling richer domain axiomatizations for robotics applications.
Finally, a version of \golog\ was recently embedded in our model of belief  \cite{Belle:2015ab}, in the style of knowledge-based programming \cite{reasoning:about:knowledge,DBLP:journals/tocl/Reiter01}. \smallskip

A major criticism leveled at much of the work in cognitive robotics \cite{lakemeyer2007chapter}, and logic-based knowledge representation more generally is that the languages are far removed from the kind of uncertainty and noise seen in  machine learning and robotics applications. A formal language such as the one considered in this article addresses this concern. It also shows the advantages of appealing to logical machinery: firstly, in admitting  natural, rich and intuitive physical laws; 
secondly, in allowing belief specifications that can exploit the full power of logical connectives, thereby going considerably beyond standard probabilistic formalisms; and thirdly,  
in alleviating the burden of determining how these probabilistic beliefs are affected in dynamical systems. In the long term, we hope it takes steps towards a general-purpose   epistemologically-adequate   representation language 
as envisioned by  McCarthy and Hayes.

%
%
%
%
%
%
%
%
%
%

%

%

%

%

%
%
%
%
%
%
%
%
%
%
%

%

% \bibliographystyle{elsarticle-num}
% \bibliography{main}
%

%
%
%

%

%

%
%
%

%

%

\end{document}